\documentclass{article}

\usepackage{reveal_preprint}
\usepackage{iftex}
\ifPDFTeX
  \usepackage{times}
\else
  \usepackage{fontspec}
\fi
\usepackage{amsmath,amssymb,amsthm,mathtools}
\usepackage{booktabs}
\usepackage{array,tabularx}
\usepackage{graphicx}
\usepackage{microtype}
\usepackage{xcolor}
\usepackage{hyperref}
\usepackage{url}
\usepackage{xurl}
\usepackage[nameinlink,noabbrev]{cleveref}
\hypersetup{hidelinks}

\newtheorem{theorem}{Theorem}
\newtheorem{lemma}{Lemma}
\newtheorem{proposition}{Proposition}
\newtheorem{corollary}{Corollary}

\newcommand{\TV}{\operatorname{TV}}
\newcommand{\KL}{D_{\mathrm{KL}}}
\newcommand{\Ber}{\operatorname{Ber}}
\newcommand{\E}{\mathbb E}
\newcommand{\Pp}{\mathbb P}
\newcommand{\one}{\mathbf 1}
\newcommand{\cA}{C_A}
\newcommand{\cH}{C_H}
\newcommand{\cF}{C_F}
\newcolumntype{Y}{>{\raggedright\arraybackslash}X}
\newcolumntype{P}[1]{>{\raggedright\arraybackslash}p{#1}}

\title{Revealing After Overwriting:\\An Exponential POMDP OPE Lower Bound\\under History-Dependent Logging}
\author{Youyu Luo\textsuperscript{1} \quad Pengzhan Zhou\textsuperscript{1,*} \quad Zhida Qin\textsuperscript{2} \quad Jia Wang\textsuperscript{1}\\
\bfseries Zuotao Fu\textsuperscript{1} \quad Yu Liu\textsuperscript{3} \quad Chao Chen\textsuperscript{1}\\[3pt]
\normalfont\small \textsuperscript{1}Chongqing University \quad \textsuperscript{2}Beijing Institute of Technology\\
\normalfont\small \textsuperscript{3}The Hong Kong Polytechnic University\\
\normalfont\small \textsuperscript{*}Corresponding author: \href{mailto:pzzhou@cqu.edu.cn}{pzzhou@cqu.edu.cn}}
\hypersetup{
  pdftitle={Revealing After Overwriting: An Exponential POMDP OPE Lower Bound under History-Dependent Logging},
  pdfauthor={Youyu Luo, Pengzhan Zhou, Zhida Qin, Jia Wang, Zuotao Fu, Yu Liu, Chao Chen}
}

\begin{document}
\maketitle

\begin{abstract}
Multi-step revealing can make off-policy evaluation tractable under
memoryless logging. With history-dependent logging, state decodability and
target-relevant evidence can separate. For every horizon $H\ge3$, we construct
two exactly realizable POMDPs with four actions, at most four states per layer,
a known logger, and a memoryless target. Action overlap, history coverage,
and observation-only revealing remain bounded independently of $H$, yet
the target values differ by $1/2$ and the KL divergence between the logged laws
is $\Theta(4^{-(H-1)})$, forcing exponential sample complexity. Logger memory
makes states distinguishable, while reset erases the model-distinguishing
evidence preserved by the target. A separate construction retains this barrier with common, known
observation-only revealing operators. Under action and history coverage,
we give a finite-class OPE guarantee using common observable value representations
that remain valid at every history. The sample bound depends polynomially on
their second-moment cost. In the common-operator construction, the same value direction has constant
marginal decoding cost but exponential history-conditioned cost. Finally,
on a fixed four-action continuum, we derive matching passive and budgeted
readout rates. With one known channel and unit read cost, early reads are optimal.
With unknown sensor bias, early reads alone remain exponentially costly.
Combining them with post-reset calibration gives sample complexity independent
of $H$ when both read types receive fixed positive expected budgets per trajectory.
\end{abstract}

\section{Introduction}

Off-policy evaluation (OPE) asks for the value of a target policy using data
generated by another policy.  In a fully observed Markov decision process,
coverage expresses whether the log contains the state--action evidence needed
by the target.  Partial observability makes the question two-sided: past
histories must span latent directions, while future outcomes must reveal them.
This past--future view underlies observable-operator and predictive-state
approaches to controlled processes \citep{littman2001predictive,hamilton2014compressed}.

\citet[Table~1, p.~2]{zhang2025tractability} organize finite-horizon POMDP OPE by behavior
memory and by single- versus multi-step revealing.  Their model-based theory
covers memoryless logging, but leaves the cell with a history-dependent
behavior policy and multi-step revealing open, conjecturing it intractable.
The open cell looks favorable one assumption at a time: actions may have full
support, behavior histories may cover every latent direction, and the
behavior-generated future may decode the current state with a bounded inverse.

The failure can be seen along one trajectory. The models differ only in a
latent bit created by a preserve action at step $1$, which the target keeps
until terminal reward. The logger preserves with probability $1/4$ per
preterminal step.
Its action-history memory is correlated with the bit and biases the reset
command, making the next observation informative about the pre-reset state.
Yet the reset erases the model difference: subsequent observable laws are
identical across models. Revealing certifies state decodability within a
model, not preservation of evidence distinguishing models. That evidence
survives only on the all-preserve event, of probability $4^{-(H-1)}$.

\newpage
\begin{table*}[t]
\centering
\small
\caption{Relationship to \citet[Thms.~5 and 7]{zhang2025tractability}.
Exact realizability means $M^\star\in\mathcal M$; multi-step and single-step
refer to their Assumptions D and E. Section and theorem numbers in the
Zhang--Jiang rows refer to that paper.}
\label{tab:comparison}
\renewcommand{\arraystretch}{1.04}
\begin{tabularx}{\textwidth}{@{}P{0.19\textwidth}P{0.18\textwidth}P{0.27\textwidth}Y@{}}
\toprule
Result and realizability & Policies $\pi_b/\pi_e$ & Revealing status & Statistical statement \\
\midrule
Zhang--Jiang Thm.~5 (exact)
& Memoryless / history-dependent allowed
& Multi-step: $M^\star$ and retained candidates
& Filtered MLE: polynomial upper bound \\
\addlinespace[1pt]
Zhang--Jiang Thm.~7 (observable-equivalent)
& Uniform memoryless / always $R$
& Single-step: $M^\star$ yes; candidates no
& Sec.~4 MLE: no Thm.~4 guarantee, with or without filtering \\
\addlinespace[2pt]
\textbf{This work} (exact)
& Known history-dependent / memoryless
& Multi-step: both candidates, $O(1)$ in $H$; also observation-only
& \textbf{Any estimator:} $\Omega(4^{H-1})$ trajectories at error $1/8$ \\
\bottomrule
\end{tabularx}
\end{table*}

\begin{figure*}[t]
\centering
\includegraphics[width=\textwidth]{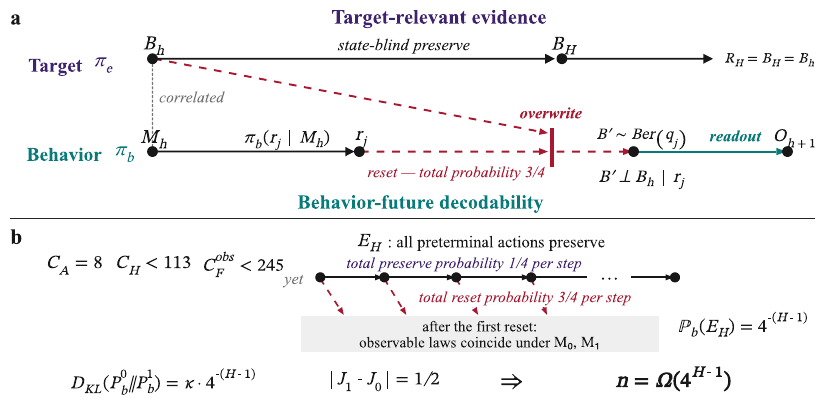}
\caption{\textbf{Overwrite--reveal mismatch.}
\textbf{a,} The logger observes $M_h$, not $B_h$.  On reset (total probability
$3/4$), memory biases $r_j$ and the environment overwrites $B_h$ by $B'$.
This channel keeps future observations informative about the pre-reset state.
The model difference appears in the log only on the all-preserve event.
The state-blind target preserves $B_h$ until the terminal reward.
\textbf{b,} The three coefficients are bounded uniformly in $H$, yet the
behavior-law KL is $\kappa\,4^{-(H-1)}$ for fixed $\kappa>0$ and
$|J_1-J_0|=1/2$. Uniform $1/8$-accurate estimation with success probability
$2/3$ requires $n=\Omega(4^{H-1})$ trajectories; here $J_m=J_{\mathsf{M}_m}(\pi_e)$.}
\label{fig:mechanism}
\end{figure*}

Theorem~\ref{thm:main} gives an exponential lower bound for every estimator,
even with exact realizability and constant action overlap, history coverage,
and observation-only revealing. On a fixed continuous family with the same
remaining kernels, the matching minimax rate is
$\Theta(4^{H-1}\varepsilon^{-2})$.

Knowing the marginal decoder is not enough: a separate eight-action
construction has common, known observation-only operators and the same
exponential barrier. For finite model classes under action and history coverage,
we identify a sufficient condition. Common observable representations of candidate
target values support OPE when they remain correct after conditioning on
each history and have bounded second moments. In the common-operator pair,
the same target continuation value can be decoded at constant marginal cost but requires
exponential cost under history conditioning. This locates the gap in the
value direction itself, without requiring full-state reconstruction.

Another route is to record evidence before overwrite. On the four-action
continuum, we derive the budget--noise tradeoff for one optional read per
trajectory, including matching bounds for adaptive rules. With one fixed
known channel and unit cost, an early-read rule can simulate the complete
record of any delayed rule at the same expected cost.

A pre-reset read is not by itself enough when the sensor has an unknown
bias. Early reads then confound the seed probability with that bias, and the
exponential barrier remains. A read after reset cannot recover the old bit,
but it can calibrate the sensor because the reset distribution is known.
We prove a matching budget law for combining the two read types: post-reset
data make the recorded target evidence usable.

\section{Setting and revealing assumptions}

We first consider a finite-horizon partially observable Markov decision
process (POMDP) from a finite realizable class
$\mathcal M$.  Write $[H]=\{1,\ldots,H\}$,
$\tau_h=(O_1,A_1,\ldots,O_h,A_h)$, and $\tau_0=\varnothing$.  A complete
logged trajectory is $(O_1,R_1,A_1,\ldots,O_H,R_H,A_H)$.  Let
$P_b^M$ denote its law under model $M$ and the known behavior policy $\pi_b$,
and let $J_M(\pi)$ denote the value of policy $\pi$ in $M$.  The behavior
policy may depend on $(\tau_{h-1},O_h)$, while OPE estimates $J_M(\pi_e)$ from
independent trajectories generated by $\pi_b$.  We use Assumptions A, C, and
D of \citet[pp.~3--4]{zhang2025tractability}.

Assumption A requires known propensities and
\begin{equation}
 \cA:=\max_{h\in[H],\tau_{h-1},o_h,a_h}
 \pi_b(a_h\mid\tau_{h-1},o_h)^{-1}<\infty.
 \label{eq:action-overlap}
\end{equation}
Let $b_S(\tau_{h-1})=\Pp(S_h\mid\tau_{h-1})$ be the pre-observation belief.
The history-coverage Gram is
\begin{equation}
  \Sigma_{H,h}=\E_{\pi_b}
  [b_S(\tau_{h-1})b_S(\tau_{h-1})^\top].
  \label{eq:history-gram}
\end{equation}
The behavior future is
$F_h=(O_h,A_h,\ldots,O_{H-1},A_{H-1},O_H)$. Define its matrix by
$[U_{F,h}]_{f,s}=\Pp^{\pi_b}(F_h=f\mid S_h=s)$. The revealing Gram is
\begin{equation}
 \Sigma_{F,h}=U_{F,h}^\top
 \operatorname{diag}(U_{F,h}\one)^{-1}U_{F,h}.
 \label{eq:future-gram}
\end{equation}
All matrices are evaluated modelwise.  The coverage and revealing
coefficients are
\begin{equation}
 \cH:=\max_{h\in[H-1]}\lambda_{\min}(\Sigma_{H,h})^{-1},
 \qquad
 \cF:=\max_{h\in[H-1]}\|\Sigma_{F,h}^{-1}\|_1.
 \label{eq:source-coefficients}
\end{equation}
Here and below, $\|\cdot\|_1$ is the induced matrix $1$-norm.
Diagonal inverses are taken on positive outcome support, equivalently after
omitting zero rows.  For the state-prior-weighted alternative in Appendix
G.2 of \citet{zhang2025tractability}, let
$p_h(s)=\Pp^{\pi_b}(S_h=s)$ be the behavior-state marginal and set
$D_{p_h}=\operatorname{diag}(p_h)$. Define
\begin{equation}
 \Sigma^{p_h}_{F,h}
 =D_{p_h}U_{F,h}^{\top}
 \operatorname{diag}(U_{F,h}p_h)^{-1}U_{F,h},
 \qquad h\in[H-1].
 \label{eq:weighted-future-gram}
\end{equation}
Write $C_F^p:=\max_{h\in[H-1]}\|(\Sigma^{p_h}_{F,h})^{-1}\|_1$.

We also use the observation-only suffix
$G_h=(O_h,\ldots,O_H)$, obtained by marginalizing the future actions in $F_h$.
Replacing $F_h$ by $G_h$ in \cref{eq:future-gram} defines
$U^{\rm obs}_{F,h}$ and $\Sigma^{\rm obs}_{F,h}$, with coefficient
$C_F^{\rm obs}:=\max_{h\in[H-1]}\|(\Sigma^{\rm obs}_{F,h})^{-1}\|_1$.
The logger still executes
its policy, so $G_h$ is the observation marginal of the same passive suffix
rather than a fresh controlled experiment.  This operator asks whether the
observations themselves reveal the current state.

With memoryless logging, the future is independent of the past given $S_h$:
\[
 \Pp^{\pi_b}(F_h=\cdot\mid\tau_{h-1})
 =U_{F,h}b_S(\tau_{h-1}).
\]
For a history-dependent logger, this identity need not hold.  Each column of
$U_{F,h}$ averages over past histories conditional on $S_h$, whereas
conditioning on a specific history fixes the logger's memory. In the
construction below, take $H=3$, model $\mathsf{M}_0$, and the history
$\tau_1=(\mathrm{start}_1,c_0)$ after the initial observation and a preserve
action. Write $1_3$ for the terminal observation of bit one. Then
\begin{equation}
 \Pp^{\pi_b}(O_3=1_3\mid\tau_1)=0.300,
 \qquad
 \bigl[U^{\rm obs}_{F,2}b_S(\tau_1)\bigr](\{O_3=1_3\})=0.388.
 \label{eq:conditioning-gap}
\end{equation}
The actual history fixes memory to $0$; the operator mixes memories again
within each state. Appendix~\ref{app:legality} gives the three-step calculation.

\section{The overwrite--reveal construction}
\label{sec:construction}

At layer one there is a single state. At later layers, $\mathsf{H}_b$ and $\mathsf{R}_b$
denote hidden and reset states with bit $b\in\{0,1\}$. The action set is
$\{c_0,c_1,r_0,r_1\}$.  Hidden states emit a blank before the terminal layer;
$\mathsf{R}_0$ and $\mathsf{R}_1$ emit $\Ber(1/10)$ and $\Ber(9/10)$, respectively.  For $h<H$,
events occur in the order
$S_h\to O_h\to R_h\to A_h\to(S_{h+1},M_{h+1})$.
At $h=H$, the observation reveals the current bit and the reward equals it.
Both precede the terminal action, drawn from the common uniform kernel.

The candidates differ only in the step-$1$ preserve transition:
\begin{equation}
\begin{array}{c|cc}
 &c_0&c_1\\ \hline
\mathsf{M}_0:\ \Pp(\mathsf{H}_1\mid\star,c_k)&1/10&2/5\\
\mathsf{M}_1:\ \Pp(\mathsf{H}_1\mid\star,c_k)&3/5&9/10.
\end{array}
\label{eq:seed}
\end{equation}
For every $2\le h<H$, each preserve action maps both $\mathsf{H}_b$ and $\mathsf{R}_b$
deterministically to $\mathsf{H}_b$.  At every preterminal layer, including $h=1$,
reset $r_0$ creates $\mathsf{R}_1$ with probability $1/10$, while $r_1$ does so with
probability $9/10$; the complementary outcome is $\mathsf{R}_0$.  Hence preserve keeps
the physical bit while returning to the hidden block. A reset draws a fresh
bit from the command-specific distribution and places it in the reset block.
The complete transition and
timing specification is in \cref{app:legality}.

The logger starts with $M_1=-1$ and maintains $M_h\in\{-1,0,1\}$ as a
deterministic function of its past observed actions. At $h=1$, a command
$c_j$ or $r_j$ writes $M_2=j$;
thereafter preserve actions keep memory and reset $r_j$ writes $M_{h+1}=j$.
It uses kernels
\begin{equation}
\begin{array}{c|cccc}
 &c_0&c_1&r_0&r_1\\ \hline
U   &1/8&1/8&3/8&3/8\\
K_0 &1/8&1/8&1/2&1/4\\
K_1 &1/8&1/8&1/4&1/2.
\end{array}
\label{eq:logger}
\end{equation}
It uses $U$ at memory $-1$ and $K_m$ at memory $m$.  Both memory states
preserve with probability $1/4$ and differ only in how they split the reset
probability.  This makes the all-preserve event exponentially rare while
carrying the memory--state correlation into the next observation.
The logger records the commanded reset label, not the unobserved reset
outcome.  It is therefore
a known history-dependent policy with no latent-state access.  At every
preterminal layer, the memoryless target chooses $c_0,c_1$ uniformly and never
resets.  At $h=H$, both policies use the same uniform action kernel.

\section{Main lower bound}

\begin{theorem}[Passive observation-only overwrite lower bound]
\label{thm:main}
For every $H\ge3$, the construction above is a realizable two-model POMDP
class with at most four states per layer and four actions, a known
history-dependent behavior policy, and a memoryless target.  The relevant
overlap, coverage, and revealing coefficients are all bounded by universal
constants independent of $H$; explicitly,
\begin{gather}
 \cA=8,\qquad
 \cH\le\frac89(67+\sqrt{3589})<113,\nonumber\\
 \max_{h\in[H-1]}\| (\Sigma^{\rm obs}_{F,h})^{-1}\|_1
 \le\frac{15625}{64}<245. \label{eq:constants}
\end{gather}
The original Assumption D, which includes future actions, also has
$\cF<245$, and
\begin{equation}
 \max_{h\in[H-1]}\|(\Sigma^{p_h}_{F,h})^{-1}\|_1<800.
 \label{eq:weighted-constant}
\end{equation}
For $i\in\{0,1\}$, write $P_b^i:=P_b^{\mathsf{M}_i}$ and let
$\Pp_i:=(P_b^i)^{\otimes n}$ denote the joint law of $n$ i.i.d. logged
trajectories.
Nevertheless,
\begin{equation}
 J_{\mathsf{M}_0}(\pi_e)=\tfrac14,\qquad J_{\mathsf{M}_1}(\pi_e)=\tfrac34,
 \label{eq:value-gap}
\end{equation}
and
\begin{equation}
 \KL(P_b^0\Vert P_b^1)=\kappa4^{-(H-1)},
 \qquad \kappa=\tfrac14\log\tfrac{27}{2}.
 \label{eq:kl-main}
\end{equation}
If $n\le4^{H-1}/(8\kappa)$, then
\begin{equation}
 \inf_{\widehat J}\sup_{i\in\{0,1\}}
 \Pp_i\!\left(
 |\widehat J-J_{\mathsf{M}_i}(\pi_e)|>\tfrac18
 \right)\ge\tfrac38.
 \label{eq:risk-main}
\end{equation}
\end{theorem}

\begin{corollary}[Negative answer to the open Table~1 cell]
\label{cor:open-cell}
Exact realizability together with the Zhang--Jiang Assumptions A, C, and D does not
imply a sample bound polynomial in $H$, $\log|\mathcal M|$, $\varepsilon^{-1}$,
$C_A$, $C_H$, and $C_F$ for passive model-based OPE with a known
history-dependent logger.  This remains true for $|\mathcal M|=2$, a
memoryless target, and bounded observation-only and state-prior-weighted
revealing coefficients.
\end{corollary}
\begin{proof}
Apply \cref{thm:main} with its horizon-independent constants.
\end{proof}

Hardness for a memoryless target already rules out a polynomial guarantee
for the broader target class in the open cell. The proof separates what the
logger makes predictable from what remains learnable from the log.

\subsection{Rare-survival experiment}

The target never resets before the terminal layer, so averaging the two entries in each row of
\cref{eq:seed} gives \cref{eq:value-gap}.  Under behavior, total preserve
probability is $1/4$ at every preterminal step.  After the first reset the two
observable experiments are identical: preceding hidden observations are
blank, and every later kernel is common.  Thus model information survives
only on
\begin{equation}
 \begin{aligned}
 E_H&=\{\text{all preterminal actions preserve}\},\\
 \Pp_b(E_H)&=4^{-(H-1)}.
 \end{aligned}
 \label{eq:rare-event}
\end{equation}
Conditioned on $E_H$, the first preserve label is uniform, so the KL chain
rule gives
\begin{align}
 \KL(P_b^0\Vert P_b^1)
 &=\frac{4^{-(H-1)}}2\left[
 \operatorname{kl}(1/10\Vert3/5)
 +\operatorname{kl}(2/5\Vert9/10)\right]\nonumber\\
 &=\kappa4^{-(H-1)}.
 \label{eq:rare-kl}
\end{align}

\subsection{Past coverage and the public anchor}

The previous action reveals whether the present state is hidden or reset,
so the history Gram is block diagonal. Reset histories cover the reset block;
at later cuts, a reset followed by a preserve creates distinct beliefs that
cover the hidden block. The worst case is the first hidden layer; its exact matrix in
Appendix~\ref{app:coverage} gives the $\cH$ bound in \cref{eq:constants}.

More importantly, under either candidate model public memory remains anchored
to the hidden bit:
\begin{equation}
 \eta_h:=\Pp(M_h=1\mid \mathsf{H}_1)-\Pp(M_h=1\mid \mathsf{H}_0)\ge\frac25.
 \label{eq:anchor}
\end{equation}
At $h=2$, direct Bayes calculation gives $\eta_2=2/5$.  For later cuts, let
$d_t=\Pp(B_t=1\mid M_t=1)-\Pp(B_t=1\mid M_t=0)$ in the complete hidden/reset
population.  Then $d_2=27/40$ and $d_{t+1}=d_t/4+3/5$.  Because preserve
selection is independent of memory, the hidden block inherits a constant gap.
The full calculation appears in \cref{app:coverage}.

\subsection{Observation-only revealing under logger-controlled resets}

For two distinct probability columns, the inverse posterior-Gram norm is
at most the inverse squared TV separation (\cref{eq:two-column}).
At intermediate hidden-state cuts $2\le h\le H-2$, the reset/emission channel yields
$\Ber(9/50)$ under $r_0$ and $\Ber(41/50)$ under $r_1$.
Their TV gap is $16/25$.  The logger kernels differ by $1/4$, and
\cref{eq:anchor} supplies the remaining $2/5$.  After marginalizing all future
action coordinates,
\begin{align}
 &\TV\!\left(
 \mathcal L(O_{h+1}\mid \mathsf{H}_0),
 \mathcal L(O_{h+1}\mid \mathsf{H}_1)
 \right)\nonumber\\[-2pt]
 &\qquad\ge \underbrace{\frac{16}{25}}_{\text{reset/emission}}
 \underbrace{\frac14}_{\text{logger}}
 \underbrace{\frac25}_{\text{anchor}}
 =\frac8{125}.
 \label{eq:revealing-factorization}
\end{align}
Reset states are separated by their current observations with TV $4/5$, and
the hidden and reset blocks have disjoint current-observation support.
Together with \cref{eq:two-column}, these facts prove the observation-only
bound in \cref{eq:constants}; the last preterminal cut, $h=H-1$, has larger separation.
For the original future operator, restoring the action coordinates refines the
posterior experiment within each two-state block and cannot decrease its
nontrivial eigenvalue.  This proves Assumption D for the constructed family.
Within each block, state probabilities stay bounded away from zero. A similarity transform
then gives the state-prior-weighted bound $<798$
(Appendix~\ref{app:revealing}).

\subsection{Testing reduction}

For $n\le4^{H-1}/(8\kappa)$, product additivity and Pinsker give
\begin{equation}
 \TV((P_b^0)^n,(P_b^1)^n)
 \le\sqrt{n\kappa4^{-(H-1)}/2}\le\frac14.
\end{equation}
Thresholding an estimator that is $1/8$-accurate under both candidates at
$1/2$ would test the two candidates. Le Cam's inequality makes the two testing
errors sum to at least $3/4$, so under at least one candidate the estimation
failure probability is at least $3/8$.
This proves \cref{thm:main}.

The revealing calculation compares states within one model; the KL compares
logs across models. Logger memory makes states distinguishable, while reset
erases the model-distinguishing evidence preserved by the target.
Appendix~\ref{app:mechanism-family} factors these two
quantities for a parametric family.

\paragraph{Known observation-only operators.}
The same distinction persists when the estimator knows the observation-only
marginal operators exactly. Appendix~\ref{app:common-operators} gives a separate
eight-action construction for this purpose.
\begin{proposition}[Common known observation-only operators]
\label{prop:known-operators}
For every $H\ge3$, there are two realizable POMDPs with at most four states
per layer and eight actions, a known history-dependent logger, and a
memoryless target whose values differ by $1/2$. Action overlap, history
coverage, and revealing remain bounded independently of $H$, including the
observation-only and state-prior-weighted versions. In fixed state and
observation coordinates, the two models share exactly the same matrices
$(U_{F,h}^{\rm obs})_{h=1}^{H-1}$.
Even given these matrices and common left inverses, every estimator needs
$\Omega(4^{H-1})$ trajectories for error at most $1/8$ with probability at
least $2/3$ uniformly over the two models.
\end{proposition}
Thus the ability to decode a state from a marginal future does not supply the
target evidence missing from the logged experiment. In this variant, averaging
over the first preserve command cancels the model parameter from the
state--memory law. The target instead fixes one command, so its value retains
the model difference.

\section{When decoding supports evaluation}
\label{sec:value-representation}

A marginal decoder averages over histories correlated with the state.
Fixing a history also fixes the logger's memory, so that decoder need not
remain correct. A sufficient alternative is to represent the target value
while retaining this conditioning.

Consider a known finite class $\mathcal M$ with common state coordinates.
Here we also allow action-dependent rewards $R_h=r_h(O_h,A_h)\in[0,1]$,
evaluated after $A_h$. All models share these known reward functions, a known
logger, and a known memoryless target. The complete observation--action log
includes $A_H$ and determines all rewards, so their coordinates may be omitted.
Write $V_h^N(s)$ for the target continuation
value in candidate $N$. For every $N$ and $2\le h\le H$, suppose an observable
function $g_h^N(\tau_{h-1},F_h)$ satisfies
\begin{equation}
 \E_M[g_h^N\mid S_h=s,\tau_{h-1}=\tau]=V_h^N(s)
 \quad\text{for every }M\in\mathcal M,
 \label{eq:value-calibration}
\end{equation}
on positive-probability conditioning events. The same function for candidate
$N$ must work across all generating models $M$. Its global second-moment
cost is
\begin{equation}
 C_V=\max_{M,N\in\mathcal M,\,2\le h\le H}\E_M[(g_h^N)^2].
 \label{eq:value-cost}
\end{equation}
Unlike full-state decoding, this condition represents only candidate target
values. It builds on future-dependent value functions
\citep{uehara2023future,zhang2024curses}.
Write $\mathsf h(P,Q)^2=\sum_z(\sqrt{P(z)}-\sqrt{Q(z)})^2$ for squared
Hellinger distance.

\begin{theorem}[History-conditioned value representations]
\label{thm:value-calibration}
Suppose the true model belongs to $\mathcal M$, the logger's inverse action
probabilities are bounded by $C_A$, and $C_V<\infty$. In every candidate
model, require $\lambda_{\min}(\Sigma_{H,h})\ge C_{\rm hist}^{-1}$ for
$2\le h\le H$, including the terminal cut. For all $M,N\in\mathcal M$,
\begin{equation}
 |J_M(\pi_e)-J_N(\pi_e)|
 \le 2HC_{\rm hist}C_A(1+2\sqrt{C_V})\,
 \mathsf h(P_b^M,P_b^N).
 \label{eq:value-law}
\end{equation}
Consequently, given $n$ independent logged trajectories,
the value of the maximum-likelihood candidate has error at most
$\epsilon$ with probability at least $1-\delta$ if
\begin{equation}
 n\ge
 \frac{40H^2C_{\rm hist}^2C_A^2(1+C_V)}{\epsilon^2}
 \log\frac{|\mathcal M|}{\delta}.
 \label{eq:value-rate}
\end{equation}
The estimator need not be given the functions $g_h^N$.
\end{theorem}

Conditional calibration and history coverage convert observable Bellman
residuals into statewise errors. The moment bound controls their sum
(Appendix~\ref{vc:appendix}): low representation cost prevents nearly
indistinguishable logs from having widely separated target values.

For finite state, action, and observation spaces, \eqref{eq:value-calibration}
is a linear system. Its column space determines feasibility; a second-order
cone program gives the smallest cost~\eqref{eq:value-cost}
(Proposition~\ref{prop:value-certificate}). Both are specified by the candidate
class. The linear system and the conic program can each have size exponential in $H$.

In the common-operator pair of Proposition~\ref{prop:known-operators},
every noninitial continuation value is the physical bit. A common marginal
value decoder has second moment below $2660$, yet the smallest cost in
\eqref{eq:value-cost} is $\Theta(4^{H-2})$. Once a preserve history is fixed,
its bit is revealed only if every remaining action preserves. A
conditionally correct decoder must compensate for this rare event
(Proposition~\ref{prop:value-gap}). The gap concerns the same value direction,
not extra state coordinates.

The condition can hold without single-step revealing. In Appendix~\ref{app:short_probe},
a logged probe at a hidden layer preserves the bit until the next observation reveals it.
Inverse-propensity weighting yields a polynomial-cost representation despite
rank-one observations at hidden layers. The next section instead adds an extra
readout to the original log.

\section{Recording and interpreting target evidence}
\label{sec:readout}

\subsection{A fixed family and its passive rate}
\label{sec:sharpness}

In the four-action construction of Section~\ref{sec:construction}, fix
$\theta\in[0,1/10]$ and replace the first preserve probabilities by
\begin{equation}
 p_{\theta,0}=\tfrac15+\theta,\qquad
 p_{\theta,1}=\tfrac7{10}+\theta.
 \label{eq:continuum}
\end{equation}
All other kernels remain fixed, and $J_\theta(\pi_e)=9/20+\theta$.
The fixed seed gap $1/2$ keeps the memory--state anchor bounded away from zero
uniformly in $\theta$. The coefficients satisfy
$C_A=8$, $C_H\le64$, and $C_F^{\rm obs},C_F\le625/4$;
Appendix~\ref{app:continuum-certificates} gives the certificate calculation.

Let $n_H^*(\varepsilon,\delta)$ be the least sample size for uniform error
$\varepsilon$ and failure probability $\delta$ on this family.

\begin{theorem}[Family-specific sharpness]
\label{thm:sharpness}
For $H\ge3$, confidence level $2/3$, and $0<\varepsilon\le1/40$, the minimax
sample complexity on the family in~\eqref{eq:continuum} satisfies
\begin{equation}
 n_H^*(\varepsilon,1/3)
 =\Theta(4^{H-1}\varepsilon^{-2}).
\end{equation}
\end{theorem}

The lower bound compares two seed probabilities separated by $4\varepsilon$.
For the upper bound, rewards on the observable all-preserve event follow the
target law. Inverse-probability weighting gives the unbiased estimator
in~\eqref{eq:ht}; Appendix~\ref{app:sharpness} proves its matching rate.

\subsection{Readout timing and sensor calibration}

A readout before the first reset can record information about the bit the target needs.
We quantify its value with a known channel, then ask what changes when the
sensor has an unknown bias.
Throughout, the family in~\eqref{eq:continuum}, logger, target, and distribution
of the original log remain unchanged. Only the extra measurement protocol varies.

\paragraph{Known-channel pre-reset reads.}
Each trajectory allows at most one extra measurement, after $O_h$ and before
$A_h$ at $2\le h<H$, provided all preceding actions preserve. A known finite
menu may depend on public history. Option $\alpha$ costs $c_\alpha\ge0$ and
reads the current bit $B$, returning $Y\in\{0,1\}$ with reliability
$\xi_\alpha\in[-1,1]$:
$\Pr(Y=B\mid B,\alpha)=(1+\xi_\alpha)/2$.
Readout choices use only public history and private randomness. Each selected
channel uses fresh randomness independent of the original process and the
selection randomization. The time, option, and result are logged; the logger
ignores them and continues unchanged.

\paragraph{A simple budget law.}
Suppose the only read option has reliability $\xi$ and unit cost at every
eligible history. For $n$ trajectories and total expected readout budget $A$,
the largest Fisher information at the parameter midpoint is of order
\begin{equation}
 n4^{-(H-1)}+\xi^2\min\{A,n/4\}.
 \label{eq:readout-total}
\end{equation}
The two terms count evidence retained in two ways: the bit survives to the
terminal observation, or a read records it before a reset. With fixed nonzero
reliability and a fixed positive budget per trajectory, the sample complexity
is $\Theta(\varepsilon^{-2})$, independent of $H$.

For a general finite menu, write $q_H=4^{-(H-1)}$ and let
$\mathcal I_\star(b)$ be the largest Fisher information at $\theta=1/20$
in the augmented log (the original log plus the optional read record),
over rules with expected cost at most $b$.
It is the optimum of a finite linear program with known coefficients
(Appendix~\ref{app:readout-flow}). A known-channel read adds information only
on paths that later reset; otherwise the terminal observation already
reveals the same bit exactly.

\begin{theorem}[Budgeted recovery of target evidence]
\label{thm:readout}
For $H\ge3$ and $0<\varepsilon\le1/40$, the minimax number of trajectories
for $\varepsilon$-accurate estimation of $J_\theta$ with probability at least
$2/3$, uniformly over the family in~\eqref{eq:continuum}, is
\begin{equation}
 n_\star(\varepsilon,b)
 =\Theta\!\left(\frac{1}{\mathcal I_\star(b)\varepsilon^2}\right).
 \label{eq:readout-rate}
\end{equation}
The budget is a total expected cost of at most $nb$ under every $\theta$.
The lower bound allows readout rules to depend on all earlier augmented
trajectories. The upper bound repeats one budget-feasible rule independently.
The constants are universal.
\end{theorem}

For each fixed read rule, the complete-log likelihood ratio relative to
$\theta_0=1/20$ is $1+(\theta-\theta_0)\ell$, with $|\ell|\le4$.
This uniform bound lets midpoint information control finite-sample estimation
throughout the parameter interval (Appendix~\ref{app:readout}).

With one channel and unit cost, an optimal rule reads only at step $2$ with
probability $\min\{4b,1\}$ when the first action preserves
(Corollary~\ref{cor:readout-uniform}). More strongly, every delayed rule using
this known channel can be simulated from a step-$2$ rule with the same
expected cost and the same read probability for each first preserve label
(Lemma~\ref{lem:early-readout}). The simulation uses complete records, so
every decision based on the delayed log can also be made from the early log.

\paragraph{Pre-reset reads with an unknown bias.}
Keep the pre-reset protocol, but give the sensor a known gain and an unknown bias:
$\Pr(Y=1\mid B)=1/4+\beta+B/2$, where
$\beta\in[-1/20,1/20]$ is shared across trajectories.
The reset probabilities $(q_0,q_1)=(1/10,9/10)$ remain known.
Before a reset, the readout mean given the first preserve label $k$ is
$1/4+d_k/2+\beta+\theta/2$, where $(d_0,d_1)=(1/5,7/10)$.
On paths that later reset, these reads identify only $\beta+\theta/2$.
Under the pre-reset protocol above, the complete-log minimax rate remains
$\Theta(q_H^{-1}\varepsilon^{-2})$, even when every eligible trajectory is read.

\paragraph{Early reads and calibration.}
For calibration, allow at most one read per trajectory at step $2$, after
$O_2$ and before $A_2$, whether the first action preserves or resets.
This reads the original seed or a fresh reset bit, respectively.
The same sensor is used in both cases, with fresh noise independent of the
original process and selection randomization; the logger ignores the read.
Let $0\le u\le1/4$ and $0\le v\le3/4$ bound the expected counts of these
two read types, averaged across trajectories.

\begin{theorem}[Target evidence and calibration]
\label{thm:readout-calibration}
For $H\ge3$ and $0<\varepsilon\le1/40$, the complete-log minimax trajectory
complexity for $\varepsilon$-accurate estimation of $J_\theta$ with probability
at least $2/3$, uniformly over $(\theta,\beta)$, is
\begin{equation}
 n_\star^{\rm cal}(\varepsilon;u,v)
 =\Theta\!\left(
 \frac{1}{\varepsilon^2\left[q_H+uv/(u+v)\right]}
 \right),
 \label{eq:calibration-rate}
\end{equation}
with $0/0:=0$. Rules may adapt across trajectories, with expected read counts
at most $nu$ and $nv$ under every parameter pair. The constants are universal.
\end{theorem}

A fixed rule attains the rate by reading independently of $O_2$, with
probabilities $4u$ after preserve and $4v/3$ after reset.
Subtracting known label-dependent baselines gives means
$\beta+\theta/2$ and $\beta$, whose difference recovers $\theta/2$.
Post-reset calibration thus makes early target evidence usable.
Appendix~\ref{app:readout-calibration} gives the proofs and an equal-budget comparison.

\section{Related work}

\citet{zhang2024curses} introduced belief and outcome coverage for model-free
OPE of memoryless policies. \citet[Thm.~3]{zhang2025tractability} give an
exponential lower bound for their model-free class with history-dependent
targets, even in a known fully observed MDP. There the restriction is access
to the target only on logged histories; model-based evaluation is immediate.
Our lower bound instead applies to arbitrary estimators and uses a memoryless target.
Table~\ref{tab:comparison} compares the closest model-based results.

Concurrent work by \citet{jajoo2026exponential} addresses the same open
question with a two-state, three-action reset construction and derives
its exact two-model minimax risk.
We establish hardness even with common, known observation-only operators
whose revealing coefficients are uniformly bounded in $H$.
We also give finite-class guarantees using history-conditioned value
representations and, on a fixed family, matching readout and sensor-calibration rates.

Adding controller memory to the state is standard
\citep{andriushchenko2022inductive,simao2023safe}. In our product-memory
variant, lifted history coverage remains $O(1)$, but future coefficients
grow as $\Theta(4^{H-h})$ (Appendix~\ref{app:product-memory}).
Thus lifting does not meet the bounded-coefficient requirements of the
memoryless-logging theorem in \citet{zhang2025tractability}.

Predictive-state representation tests specify future actions as interventions
\citep{littman2001predictive,hamilton2014compressed}.  Segmented policies switch
at a checkpoint and reset suffix memory \citep{kwon2024latent}; controlled
sensing selects actions for model discrimination \citep{nitinawarat2015controlled}.
Our passive logger instead controls the suffix: marginal decoding remains
well conditioned while target-relevant evidence becomes exponentially rare.

Additional training observations can aid POMDP learning
\citep{lee2023hindsight,guo2024multiple}.
Costly observation and active learning study measurement value and the balance
between target and nuisance-parameter estimation
\citep{nam2021observation,sloman2024bayesian}.
Our finite-sample budget laws quantify this balance under overwriting:
early reads save target evidence, while reset outcomes calibrate the sensor.

Future-dependent value functions and history instruments support OPE
without full-state recovery \citep{uehara2023future,zhang2024curses}.
Their memory-dependent extensions condition on both state and memory.
Our positive result uses common cross-model representations, global second
moments, and coverage on the original state. These give a finite-class
maximum-likelihood guarantee. The common-operator pair separates marginal
and history-conditioned representation costs exponentially for the same value.
Bridge and proxy methods also identify values under other observation
structures \citep{nair2021spectral,shi2022confounded}.
Under rank, proxy, and coverage conditions, \citet{kuang2025order} give
polynomial-in-$H$ bounds for memoryless targets in confounded tabular POMDPs.
Here the logger uses only public history, so the exponential barrier arises
without latent confounding.

Revealing-POMDP lower bounds also arise in online learning
\citep{chen2023revealing}; OPE identification with latent-state-dependent
behavior was studied by \citet{tennenholtz2020ope}.

\section{Discussion and conclusion}

A history-dependent logger can reveal states while erasing the evidence that
the target preserves. This mismatch persists with common, known observation-only
operators. Under coverage, common low-cost value representations valid at every history
suffice for OPE. The hard pair makes their cost exponentially larger than
marginal decoding of the same value.

On the four-action continuum, an early-read rule with a fixed known channel and
unit cost can simulate any delayed-read rule at the same expected cost.
With unknown sensor bias, early reads alone remain exponentially costly.
Allocating fixed positive expected budgets per trajectory to both early reads
and post-reset calibration gives sample complexity independent of $H$.
Data that do not contain the original bit can therefore make earlier measurements
useful for target evaluation.

\subsection*{AI use statement}

Generative AI tools assisted with conceptual development, mathematical claims
and proofs, experiment design and analysis, code and figures, literature
search, and manuscript editing. The authors take responsibility for the final
content.

\subsection*{Reproducibility statement}

The results are proved analytically. The supplementary package contains two
independent implementations of the four-action construction, an implementation
of the parameterized family, and two independent implementations of the
product-memory variant. They use exact rational arithmetic for probabilities
and matrix calculations. Additional checks cover the common-operator construction and the
complete logs with a noisy readout. They verify timing, observable-law
construction, Gram matrices, and closed-form identities at the stated finite
horizons; the arbitrary-horizon conclusions follow from the analytic proofs.
Appendix~\ref{app:readout-calibration} supplies the unknown-bias proofs and
an equal-budget comparison, with exact checks of the complete-log distributions, simulation code, fixed
seeds, and per-repetition results in the supplementary package.
Appendices~\ref{vc:appendix} and~\ref{app:short_probe} prove the value-representation
result and its examples; the package includes exact finite-instance checks.

\bibliographystyle{iclr2027_conference}
\bibliography{references}

\appendix

\section{Complete construction and event timing}
\label{app:legality}

Let $\mathcal S_1=\{\star\}$ and
$\mathcal S_h=\{\mathsf{H}_0,\mathsf{H}_1,\mathsf{R}_0,\mathsf{R}_1\}$ for $2\le h\le H$.  Observations are
stage tagged, and the within-stage order is
$S_h\to O_h\to R_h\to A_h\to(S_{h+1},M_{h+1})$ for $h<H$.
Write $q_0=1/10$, $q_1=9/10$, and let $p_{ik}$ denote the entry in row $i$,
column $k$ of \cref{eq:seed}.  The complete preterminal transition and memory
update are
\begin{center}
\small
\begin{tabularx}{\textwidth}{@{}P{0.12\textwidth}P{0.16\textwidth}P{0.16\textwidth}YP{0.16\textwidth}@{}}
\toprule
Stage & Current state & Action & Next physical state & Next memory \\
\midrule
$h=1$ & $\star$ & $c_k$ & $\mathsf{H}_B$, $B\sim\Ber(p_{ik})$ in $\mathsf{M}_i$ & $M_2=k$ \\
$h=1$ & $\star$ & $r_j$ & $\mathsf{R}_B$, $B\sim\Ber(q_j)$ & $M_2=j$ \\
$2\le h<H$ & $\mathsf{H}_b$ or $\mathsf{R}_b$ & $c_k$ & $\mathsf{H}_b$ deterministically & $M_{h+1}=M_h$ \\
$2\le h<H$ & any state & $r_j$ & $\mathsf{R}_B$, $B\sim\Ber(q_j)$ & $M_{h+1}=j$ \\
\bottomrule
\end{tabularx}
\end{center}
Thus a preserve action from $\mathsf{R}_b$ returns to $\mathsf{H}_b$ rather than retaining the
reset-state type.  A first-stage reset is model independent and remains so
forever, because the seed table is the only model-dependent kernel.

At $h=1$ the observation is $\mathrm{start}_1$.  For $2\le h<H$, hidden
states emit $\bot_h$, while $\mathsf{R}_b$ emits a Bernoulli symbol with parameter
$1/10$ for $b=0$ and $9/10$ for $b=1$.  At layer $H$, both $\mathsf{H}_b$ and $\mathsf{R}_b$
emit $b_H$.  The common known reward is $R(o)=\one\{o=1_H\}$; all earlier
rewards vanish.  Reward is emitted before the terminal action $A_H$.  Both
policies then use the same uniform terminal kernel, which belongs to the full
trajectory law but not to $F_h$ and induces no transition.

The public memory initially equals $M_1=-1$.  The table makes it a
deterministic function of public action history under both models.  It records
the commanded reset label, never the unobserved reset outcome.

\paragraph{Marginal versus history-conditional prediction.}
For $H=3$ in $\mathsf{M}_0$, let $\tau_1=(\mathrm{start}_1,c_0)$.  Then $M_2=0$ and
the belief on $(\mathsf{H}_0,\mathsf{H}_1)$ is $(9/10,1/10)$.  The logger uses $K_0$, so
\[
 \Pp^{\pi_b}(O_3=1_3\mid\tau_1)
 =\tfrac14\tfrac1{10}+\tfrac12\tfrac1{10}+\tfrac14\tfrac9{10}
 =\tfrac3{10}.
\]
By Bayes' rule, $\Pp(M_2=1\mid S_2=\mathsf{H}_0)=2/5$ and
$\Pp(M_2=1\mid S_2=\mathsf{H}_1)=4/5$.  Averaging over these memory distributions gives
$\Pp^{\pi_b}(O_3=1_3\mid S_2=\mathsf{H}_0)=71/200$ and
$\Pp^{\pi_b}(O_3=1_3\mid S_2=\mathsf{H}_1)=137/200$.
Hence $U^{\rm obs}_{F,2}b_S(\tau_1)$ assigns probability
$(9/10)(71/200)+(1/10)(137/200)=97/250$ to $O_3=1_3$, rather than $3/10$.
Since $O_2=\bot_2$ throughout this calculation, the discrepancy already occurs
for the observation-only suffix $G_2$.

\section{Exact behavior and lower-bound proof}

\begin{lemma}[Action overlap and memory symmetry]
The behavior memory is uniform on $\{0,1\}$ after initialization.  At every
preterminal layer, each reset label has marginal probability $3/8$, total
preserve probability is $1/4$, and $\cA=8$.
\end{lemma}
\begin{proof}
The smallest entry of \cref{eq:logger} is $1/8$.  If memory is uniform,
kernel symmetry sends equal mass to the two next-memory labels: the
$+1/8$ and $-1/8$ deviations from $3/8$ cancel after averaging.  Preserve retains
memory, proving invariance and the marginal claims.
\end{proof}

\begin{lemma}[Exact target value and behavior divergence]
\Cref{eq:value-gap,eq:kl-main} hold, and
$\TV(P_b^0,P_b^1)=\frac12 4^{-(H-1)}$.
\end{lemma}
\begin{proof}
Averaging the two entries in each row of \cref{eq:seed} gives target
values $1/4$ and $3/4$, respectively.
The behavior laws coincide after the first reset, while the no-reset event has
probability $4^{-(H-1)}$.  Conditioned on it, the two terminal Bernoulli pairs
are $(1/10,3/5)$ and $(2/5,9/10)$ under equally likely first preserve labels.
The KL chain rule yields \cref{eq:rare-kl}; both Bernoulli pairs have total
variation $1/2$.
\end{proof}

\section{History coverage and public-memory anchor}
\label{app:coverage}

Write $v(p)=(1-p,p)^\top$.  The $h=2$ hidden block under $\mathsf{M}_0$ is
\[
 \frac18[v(1/10)v(1/10)^\top+v(2/5)v(2/5)^\top]
 =\frac1{800}\begin{pmatrix}117&33\\33&17\end{pmatrix}.
\]
Under $\mathsf{M}_1$ its diagonal entries are exchanged, so its eigenvalues agree.
In particular, under $\mathsf{M}_0$,
\begin{equation}
 \Sigma^{\rm hid}_{H,2}=\frac1{800}
 \begin{pmatrix}117&33\\33&17\end{pmatrix},\qquad
 \lambda_{\min}=\frac{67-\sqrt{3589}}{800}.
 \label{eq:history-certificate}
\end{equation}
The reset block is
$\frac38[v(1/10)v(1/10)^\top+v(9/10)v(9/10)^\top]$ and has minimum
eigenvalue $6/25$.  At $h\ge3$, retaining only histories produced by a reset
followed by a preserve gives hidden-block lower bound
\[
 \frac3{32}[v(1/10)v(1/10)^\top+v(9/10)v(9/10)^\top],
\]
whose minimum eigenvalue is $3/50$.  Conditioning on the intervening noisy
observation can only add posterior second moment in PSD order.  Hence
\cref{eq:history-certificate} is the uniform worst case.

For the anchor, direct Bayes calculation at $h=2$ gives
$\eta_2=2/5$ under both candidates.  For later cuts, define
$q_{t,m}=\Pp(B_t=1\mid M_t=m)$ over the complete hidden/reset population and
set $d_t=q_{t,1}-q_{t,0}$.  At $t=2$,
\[
 (q_{2,0},q_{2,1})=
 \begin{cases}
 (1/10,31/40),&\mathsf{M}_0,\\
 (9/40,9/10),&\mathsf{M}_1,
 \end{cases}
 \qquad d_2=27/40.
\]
The common kernels imply
\[
 q_{t+1,0}=\tfrac14q_{t,0}+\tfrac3{40},\qquad
 q_{t+1,1}=\tfrac14q_{t,1}+\tfrac{27}{40},
\]
so $d_{t+1}=d_t/4+3/5$. Starting from $d_2=27/40$, this sequence
increases toward $4/5$, so $d_t\ge27/40$ for every $t\ge2$.
Under uniform memory, Bayes' rule gives
\begin{align*}
 &\Pp(M=1\mid B=1)-\Pp(M=1\mid B=0)\\
 &\qquad=\frac{q_1-q_0}{(q_0+q_1)(2-q_0-q_1)}
 \ge q_1-q_0.
\end{align*}
Preserve selection is independent of memory, which proves
\cref{eq:anchor} at every hidden cut.

\section{Revealing certificates}
\label{app:revealing}

For two distinct probability columns $P_0,P_1$, the posterior Gram has
eigenvalues one and
\begin{equation}
 \lambda(P_0,P_1)=\frac12\sum_y
 \frac{(P_0(y)-P_1(y))^2}{P_0(y)+P_1(y)}
 \ge \TV(P_0,P_1)^2.
 \label{eq:two-column}
\end{equation}
Direct expansion gives the identity, and Cauchy--Schwarz gives the TV bound. A
positive-definite symmetric two-by-two row-stochastic matrix has inverse one-norm equal to the
reciprocal of its nontrivial eigenvalue.

At an intermediate hidden cut, composing reset and emission channels gives
$O_{h+1}\mid r_0\sim\Ber(9/50)$ and
$O_{h+1}\mid r_1\sim\Ber(41/50)$, whose TV is $16/25$.  Combining this with
the logger and anchor gaps proves \cref{eq:revealing-factorization}.  At $h=H-1$,
preservation reveals the original bit.  Because $\eta_h>0$ and both the reset
and emission channels are ordered from label zero to label one, the
memory-mediated reset term has the same sign, giving TV at least
$1/4+(1/5)(2/5)=33/100$.  Reset states have current-observation TV $4/5$.
The current observation gives disjoint support to hidden and reset blocks,
proving the stated observation-only inverse bound.

Within either binary block, the posterior given the observation-only future
is the conditional expectation of the posterior given the action-containing
future.  Hence the action-containing posterior Gram dominates the
observation-only Gram in PSD order.  Both are symmetric row-stochastic
two-state matrices with the same unit eigenvector, so adding action coordinates
cannot decrease the nontrivial eigenvalue.

For the state-prior-weighted alternative, recall
\cref{eq:weighted-future-gram}.  Within a current-observation block, normalize
the state prior to $q=(q_0,q_1)$ and write $D_q=\operatorname{diag}(q)$.
Current-observation support makes the full matrix block diagonal, and
normalizing the prior within a block leaves its weighted block unchanged.
The block priors are uniformly balanced.  At $h=2$, the hidden-block prior
is $(3/4,1/4)$ under $\mathsf{M}_0$ and its reversal under $\mathsf{M}_1$; the reset block is
balanced.  If $b_t=\Pp(B_t=1)$ in the complete population, then
$b_2\in\{7/16,9/16\}$ and
\[
 b_{t+1}=\tfrac14 b_t+\tfrac38.
\]
Thus every later hidden-block prior lies in $[7/16,9/16]$, while every later
reset block remains balanced.  Consequently all blocks satisfy
$q_0q_1\ge3/16$ and $q_{\max}/q_{\min}\le3$.  Put
\[
 B=D_q^{1/2}U^\top\operatorname{diag}(Uq)^{-1}UD_q^{1/2}.
\]
Its nontrivial eigenvalue obeys
\[
 \lambda_q\ge4q_0q_1\TV(P_0,P_1)^2\ge\frac{48}{15625}.
\]
The normalized weighted block is
\[
 \Sigma_F^q=D_qU^\top\operatorname{diag}(Uq)^{-1}U
 =D_q^{1/2}BD_q^{-1/2}.
\]
Hence
\[
 \|(\Sigma_F^q)^{-1}\|_1
 \le\sqrt6\lambda_q^{-1}
 \le\frac{15625\sqrt6}{48}<798.
\]
Adding future actions cannot weaken this blockwise bound.

\section{Sharpness constants}
\label{app:sharpness}

For the fixed family, $n_H^*(\varepsilon,\delta)$ is the smallest $n$ for
which some estimator satisfies
\begin{equation}
 \sup_{\theta\in[0,1/10]}
 \Pp_\theta\!\left(
 |\widehat J-J_\theta(\pi_e)|>\varepsilon
 \right)\le\delta.
 \label{eq:minimax-definition}
\end{equation}
The estimator attaining the upper bound is
\begin{equation}
 \widehat J_{\rm HT}=\frac1n\sum_{i=1}^n
 \frac{\one\{E_{H,i}\}R_{i,H}}{4^{-(H-1)}}
 \label{eq:ht}
\end{equation}
because $E_H$ has probability $4^{-(H-1)}$ and its conditional reward law
equals the target reward law.

\paragraph{Uniform certificates for the continuous family.}
\label{app:continuum-certificates}
For \cref{eq:continuum}, the $h=2$ hidden history block has determinant
$(p_{\theta,1}-p_{\theta,0})^2/64=1/256$ and trace at most $1/4$,
so its minimum eigenvalue is at least $1/64$. The argument in
\cref{app:coverage} gives the same lower bound at later cuts. Also,
\[
 \eta_2=\frac{p_{\theta,1}-p_{\theta,0}}
 {(p_{\theta,0}+p_{\theta,1})(2-p_{\theta,0}-p_{\theta,1})}\ge\frac12.
\]
The anchor recurrence in \cref{app:coverage} preserves this bound.
The factorization in \cref{eq:revealing-factorization} gives next-observation
TV at least $2/25$. Thus $C_A=8$, $C_H\le64$, and
$C_F^{\rm obs}\le625/4$ uniformly over $\theta$; posterior refinement gives
the same bound for $C_F$.

\paragraph{Matching risk bounds.}

For \cref{eq:continuum}, let $q_H=4^{-(H-1)}$.  Taking
$\theta_0=0$ and $\theta_1=4\varepsilon$ gives target gap
$4\varepsilon$.  For $\varepsilon\le1/40$, Bernoulli
$\mathrm{KL}\le\chi^2$ gives one-trajectory KL at most
$100q_H\varepsilon^2$.  Pinsker and Le Cam imply
\[
 n_H^*(\varepsilon,1/3)>
 \frac{1}{800}q_H^{-1}\varepsilon^{-2}.
\]
For \cref{eq:ht},
$\E X_i=J_\theta(\pi_e)$,
$\operatorname{Var}(X_i)\le q_H^{-1}$, and
$0\le X_i\le q_H^{-1}$.  Bernstein gives
\[
 \Pp(|\widehat J_{\rm HT}-J_\theta|>\varepsilon)
 \le2\exp(-3nq_H\varepsilon^2/8).
\]
Thus
\[
 \frac1{800}4^{H-1}\varepsilon^{-2}
 <n_H^*(\varepsilon,1/3)
 \le\frac{8\log6}{3}4^{H-1}\varepsilon^{-2}+1.
\]

\section{A parametric overwrite--reveal family}
\label{app:mechanism-family}

Fix $0<\rho<1$ and $0<\gamma<(1-\rho)/2$. In model $\mathsf{M}_i$, a preserve
action $c_k$ at step $1$ creates bit one with probability $p_{ik}\in(0,1)$,
where $p_{i1}>p_{i0}$. Reset $r_j$ creates a fresh bit with success
probability $q_j$, where $0<q_0<q_1<1$. Before the terminal step,
reset state $\mathsf{R}_b$ emits one with probability $e_b$, where $0<e_0<e_1<1$.
The two target seed averages are distinct. Initially the
logger assigns $\rho/2$ to each preserve action and $(1-\rho)/2$ to each
reset.  In public memory state $m$, it keeps probability $\rho/2$ on each
preserve action and assigns $(1-\rho)/2+\gamma$ to $r_m$ and
$(1-\rho)/2-\gamma$ to $r_{1-m}$.  All later dynamics, memory updates, the
terminal kernel, and the target are as in the main four-action construction.
Put
$\delta_q=q_1-q_0$, $\delta_e=e_1-e_0$, and
$D_*=\frac12\sum_{k=0}^1\operatorname{kl}(p_{0k}\Vert p_{1k})$.

For $2\le h\le H-2$, next-observation separation and passive model
information factor as
\begin{equation}
 \TV_{\rm next}=2\gamma|\eta_h|\delta_q\delta_e,
 \qquad
 \KL(P_b^0\Vert P_b^1)=\rho^{H-1}D_*.
 \label{eq:mechanism-box}
\end{equation}
Here $\TV_{\rm next}$ compares the laws of $O_{h+1}$ given
$S_h=\mathsf{H}_0$ or $S_h=\mathsf{H}_1$ within one model.
The KL compares complete logs across models.
At $h=H-1$, the observation TV is
$\rho+2\gamma\eta_h\delta_q$, because preservation directly exposes the old
bit.  For each fixed parameter choice satisfying the preceding strict
inequalities, the history Gram has a positive lower bound that is uniform in
$H$.

\section{A public-memory lift variant}
\label{app:product-memory}

Fixing a finite-state controller induces a standard product process on
environmental and controller states \citep{andriushchenko2022inductive}, and
observed controller memory has been used in offline POMDP analysis
\citep{simao2023safe}.  \citet[Thm.~5]{zhang2025tractability} applies after
such a lift when the lifted model class is realizable and its overlap,
history-coverage, and revealing coefficients are bounded.
The variant separates the two effects of public-memory lifting.
A writer coordinate supplies linearly independent belief vectors
within each public-memory block, keeping lifted history coverage bounded.
The target-relevant physical bit, however, affects a lifted product-POMDP
suffix only along the exponentially rare all-preserve path, so lifted future
revealing becomes exponentially ill conditioned.

\begin{proposition}[Family-specific public-memory lift separation]
\label{prop:product-memory}
For every $H\ge3$, there is an exactly realizable two-model variant of
\cref{thm:main} with at most four physical states per layer, eight actions, a
known logger with two public memory states after initialization, and a
memoryless target.  Its original marginal coefficients satisfy, uniformly in
$H$,
\begin{equation}
 C_A=32,\qquad C_H<113,\qquad
 C_F^{\rm obs},C_F<1372,\qquad C_F^p<4481.
 \label{eq:product-original}
\end{equation}
In the standard lift $\bar S_h=(S_h,M_h)$ and
$\bar O_h=(O_h,M_h)$, the logger is memoryless and the lifted history
coefficient obeys
\begin{equation}
 \bar C_H\le400.
 \label{eq:product-history}
\end{equation}
For a hidden-state cut $2\le h\le H-1$, put $L=H-h$.  The action-containing
and action-deleted lifted future operators have exact inverse coefficients
\begin{align}
 \left\|(\bar\Sigma_{F,h}^{\rm act})^{-1}\right\|_1
 &=4^L, \label{eq:product-action}\\
 \left\|(\bar\Sigma_{F,h}^{(O,M)})^{-1}\right\|_1
 &=\frac{1547}{480}\,4^L. \label{eq:product-observation}
\end{align}
Conditioning on $M_h$ and omitting all future memory observations instead
gives the conditional two-column coefficient
$\|(\Sigma_{G\mid M,h})^{-1}\|_1=4^{L+1}-16/25$; it is not the lifted
observation coefficient because future memory is part of $\bar O$. The target
values and the KL divergence between the complete logged laws are
\begin{equation}
 J_{\mathsf{M}_0}=\frac14,\qquad J_{\mathsf{M}_1}=\frac34,\qquad
 \KL(P_b^0\Vert P_b^1)
 =\frac14\log\frac{27}{2}\,4^{-(H-1)}.
 \label{eq:product-experiment}
\end{equation}
\end{proposition}

\begin{proof}
\medskip\noindent\textbf{Variant and timing.}
Retain the physical states, emissions, reset kernels, terminal reward, and
seed table of \cref{eq:seed}.  Replace each physical command by an action with
a writer bit, $\mathcal A=\{c_{k,m},r_{j,m}:k,j,m\in\{0,1\}\}$: the first
subscript is physical and the second sets $M_{h+1}=m$.  Let $M_1=-1$ and let
$a(m,x)$ equal $3/4$ if $m=x$ and $1/4$ otherwise.  The first line below is
the initial logger kernel; the second applies when the current memory is
$\ell\in\{0,1\}$:
\begin{align*}
 \pi_b(c_{k,m})&=\tfrac18a(m,k),&
 \pi_b(r_{j,m})&=\tfrac38a(m,j),\\
 \pi_b(c_{k,m}\mid\ell)&=\tfrac18a(m,\ell),&
 \pi_b(r_{j,m}\mid\ell)&=K_\ell(r_j)a(m,j),
\end{align*}
where $(K_0(r_0),K_0(r_1))=(1/2,1/4)$ and
$(K_1(r_0),K_1(r_1))=(1/4,1/2)$.  These formulas also define the logged
terminal action.  At each preterminal stage, the target assigns probability
$1/4$ to every $c_{k,m}$ and never resets; its terminal kernel is common and
memory independent.  Reward is emitted before that terminal action.  Thus
both candidates share every space and kernel except the seed table, and
$M_h$ is publicly reconstructible from past actions.

The smallest logger probability is $1/32$.  Memory is uniform after
initialization and, conditional on its current value, remains unchanged with
probability $5/8$ and flips with probability $3/8$.  After summing over the
writer bit, each physical preserve label has mass $1/8$, each initial reset
label has mass $3/8$, and total preserve probability is always $1/4$.

\medskip\noindent\textbf{Experiment and marginal certificates.}
The target averages the two physical seed labels and subsequently preserves,
which gives the values in \cref{eq:product-experiment}.  Let $E_H$ be the event
that all preterminal physical commands preserve.  Then
$\Pp_b(E_H)=4^{-(H-1)}$; after the first reset every remaining kernel is model
common.  Conditional on $E_H$, the two terminal Bernoulli pairs are
$(1/10,3/5)$ and $(2/5,9/10)$.  The KL chain rule gives
\cref{eq:product-experiment}.  Appending public memory and a common
conditional terminal-action kernel does not change this divergence.

For the original latent state, summing the writer coordinate recovers the
stage-two physical history blocks of the main construction. Under $\mathsf{M}_0$, the worst hidden
block is
\[
 \frac1{800}\begin{pmatrix}117&33\\33&17\end{pmatrix},
\]
with diagonal entries exchanged under $\mathsf{M}_1$.
Reset-then-preserve histories give the same later lower bound, proving
$C_H<113$.  In the complete hidden/reset population, let
$q_{t,m}=\Pp(B_t=1\mid M_t=m)$ and $d_t=q_{t,1}-q_{t,0}$.  Then
\[
 d_2=\frac{27}{80},\qquad d_{t+1}=\frac18d_t+\frac3{10},
\]
so $d_t\ge27/80$.  At $h=2$, the hidden-state memory anchor is $1/5$.
For $h\ge3$, the preserve writer channel transfers half of $d_{h-1}$ to the
hidden-block bit--memory contrast.  Uniform memory and Bayes' rule turn this
into a memory-given-bit anchor of at least $d_{h-1}/2\ge27/160$.  Combining this with
reset-command gap $1/4$ and reset/emission separation $16/25$ gives
next-observation TV at least $27/1000$.  The binary posterior-Gram identity
yields the intermediate hidden-block bound $10^6/729$. At the last preterminal cut,
direct preservation gives TV at least $1/4$; reset-state blocks have current
observation TV $4/5$ and disjoint support from hidden blocks.  Adding action
coordinates refines each binary posterior experiment, so
$C_F^{\rm obs},C_F\le10^6/729<1372$.  The normalized state priors obey
$q_0q_1\ge3/16$ and $q_{\max}/q_{\min}\le3$, hence
\[
 \lambda_q\ge4q_0q_1\TV^2\ge\frac{2187}{4{,}000{,}000}.
\]
The similarity transform and the same refinement argument give
$C_F^p\le4{,}000{,}000\sqrt6/2187<4481$, proving
\cref{eq:product-original}.

\medskip\noindent\textbf{Lifted history coverage.}
The lift has eight states at each noninitial layer.  At $h=2$, conditioning on
the first action separates the history Gram into memory and physical-state
blocks.  Up to bit reversal, its two hidden block types are
\[
 \frac1{3200}\begin{pmatrix}279&51\\51&19\end{pmatrix},\qquad
 \frac1{3200}\begin{pmatrix}189&81\\81&49\end{pmatrix}.
\]
Both have determinant $27/102400$ and trace at most $149/1600$, hence minimum
eigenvalue greater than $1/400$.  The $h=2$ reset blocks are, up to reversal,
$\frac1{800}\left(\begin{smallmatrix}183&27\\27&63\end{smallmatrix}\right)$
and have minimum eigenvalue greater than $1/400$.  At later cuts, for each
destination memory retain a reset of each physical label followed by any
preserve writing that memory.  The
reset event has probability at least $1/16$, as does the total mass of those
two preserve labels.  Since the reset beliefs have contrast eigenvalue
$16/25$, this gives $(16/25)/256=1/400$ for hidden blocks.  A fixed-writer
reset block receives each reset belief with mass at least $1/16$, giving
contrast eigenvalue at least $1/25$.  Omitted histories add positive
semidefinite terms, proving \cref{eq:product-history}.

\medskip\noindent\textbf{Lifted future operators in the product POMDP.}
Reset product blocks are separated by current-observation TV $4/5$, so their
inverse coefficient is at most $25/16$.  Every hidden-block coefficient
derived next is at least $4$, so the full-operator maxima come from hidden
blocks.
For
\[
 \bar F_h=((O_h,M_h),A_h,\ldots,A_{H-1},(O_H,M_H)),
\]
the old bit affects the suffix only when all $L$ physical commands preserve,
an event of probability $q=4^{-L}$.  Action coordinates separate
reset-containing from all-preserve outcomes, so within each observed-memory block the two
hidden columns have the form $P_b=(1-q)R+qQ_b$ with disjoint
$R,Q_0,Q_1$.  The posterior Gram has eigenvalues $1$ and $q$, proving
\cref{eq:product-action}.

Now fix $M_h$ and use only $G_h=(O_h,\ldots,O_H)$.  A reset at the final
decision is the only reset outcome that shares a terminal row with an
all-preserve outcome.  Direct aggregation gives the nontrivial eigenvalue of
the two-column posterior Gram
\[
 \lambda_{G\mid M}
 =\frac{4^{-(L+1)}}{1-4^{-(L-1)}/25}.
\]
Its reciprocal is the conditional two-column coefficient stated in the
proposition.  Omitting all memory observations from the four-column
product-state operator instead gives
\[
 P_{1,0}-P_{0,0}=P_{1,1}-P_{0,1},
\]
so that operator is singular.

Finally consider the action-deleted lifted product future
\[
 \bar G_h=((O_h,M_h),\ldots,(O_H,M_H)).
\]
The observed memory path identifies the memory $x$ before the last decision
and the final writer $z$.  For each $(x,z)$, the final preserve/reset
collision contributes $\phi(x,z)$ to the nontrivial eigenvalue of the hidden
two-column posterior Gram, where $p_c(x,z)$ is preserve mass and $u_y(x,z)$
is reset mass producing terminal bit $y$:
\[
 \phi(x,z)=\frac{p_c(x,z)^2}{2}
 \left[\frac1{p_c(x,z)+2u_0(x,z)}
       +\frac1{p_c(x,z)+2u_1(x,z)}\right].
\]
For $z=x$, $(p_c,\{u_0,u_1\})=(3/16,\{11/32,3/32\})$ and
$\phi=15/224$.  For $z\ne x$, these masses are
$(1/16,\{21/160,29/160\})$ and $\phi=75/7072$.  Hence
$\sum_z\phi(x,z)=120/1547$, independently of $x$.  Earlier all-preserve
memory paths have mass $4^{-(L-1)}$, giving the nontrivial eigenvalue
$4^{-(L-1)}120/1547$ and proving \cref{eq:product-observation}.

\medskip\noindent\textbf{Interpretation.}
The writer coordinate changes the next product state and hence the next
logger law, including the logged terminal-action law when written at $H-1$,
while leaving the physical-bit kernel unchanged.  It supplies the histories
needed for bounded lifted history coverage, while the all-preserve bottleneck
keeps the lifted future operator exponentially ill conditioned.
\end{proof}

\section{Finite-instance computational verification}

We evaluate the main four-action construction with a compressed implementation
for $H=3$--$9$ and an independent exhaustive implementation for $H=3$--$5$;
the latter reconstructs the full behavior law rather than substituting the
rare-path KL formula.  A parameterized implementation evaluates the anchor,
KL factorization, and distinct intermediate and terminal formulas for
$H=3$--$6$.  Two separately written product-memory implementations compute the
marginal and lifted Grams and exact behavior-law TV/KL over the respective
ranges $H=3$--$9$ and $H=3$--$5$; the exhaustive implementation also
reconstructs the full public law.  These finite-instance calculations
corroborate the analytic identities; the arbitrary-$H$ results follow from the
analytic proofs.

The common-operator variant is checked by a dynamic-programming implementation
for $H=3$--$9$ and a separate exhaustive implementation for $H=3$--$5$.
They adapt the corresponding product-memory engines to its distinct kernels
and compare both models' observation-only operators entry by entry.
For the four-action readout protocol, a separately implemented full-log
checker evaluates $H=3,4,5$ under three readout schedules and three noise
levels. It verifies the complete-log likelihood ratio, information, estimator
moments, and actual expected cost using exact rational arithmetic.

\section{Common known marginal operators}\label{app:common-operators}

We prove Proposition~\ref{prop:known-operators} with a separate eight-action
construction. The four-action construction proves the main lower bound;
the eight-action construction in Appendix~\ref{app:product-memory} examines
public-memory lifting. Here the step-$1$ preserve transition depends on both
the physical command and the memory-writing coordinate. The command carries
the model difference, while the writer maintains a bit--memory correlation.
Averaging over the command removes the
model difference from the marginal future operators but retains it in the
target value.

\subsection{Construction}\label{ko:construction}

Fix $H\ge3$ and $\omega\in\{-1/4,1/4\}$. The first state is $\star$; subsequent state spaces are $\{\mathsf{H}_0,\mathsf{H}_1,\mathsf{R}_0,\mathsf{R}_1\}$. The subscript is a physical bit. There are eight actions,
$\{c_{k,m},r_{j,m}:k,j,m\in\{0,1\}\}$. Every preterminal action writes its second subscript into the public memory $M_{h+1}=m$. Thus memory is determined by the action history, with $M_1=-1$.
The first preserve action $c_{k,m}$ produces $\mathsf{H}_B$, where
\begin{equation}\label{ko:seed}
 B\sim\Ber(p_{\omega,k,m}),\qquad
 p_{\omega,k,m}=\frac12+\frac{2m-1}{8}+\omega(2k-1).
\end{equation}
Every later preserve action maps either $\mathsf{H}_b$ or $\mathsf{R}_b$ to $\mathsf{H}_b$. At every preterminal step, $r_{j,m}$ instead produces $\mathsf{R}_B$ with $B\sim\Ber(q_j)$, independently of the old bit, where $(q_0,q_1)=(1/10,9/10)$. Only the first-preserve kernel in~\eqref{ko:seed} differs between models.

Observations include their time index. Initially $O_1=\mathrm{start}_1$. Before the terminal step, $\mathsf{H}_b$ emits a blank and $\mathsf{R}_b$ emits a Bernoulli observation with parameter $q_b$. At step $H$, all states reveal their bit exactly. Rewards are zero until $R_H=\one\{O_H=1_H\}$. The order is observation, reward, action, then transition; $A_H$ occurs after the terminal reward and causes no further transition.

Write $a(m,x)=3/4$ if $m=x$ and $1/4$ otherwise, and let
$(K_0(r_0),K_0(r_1))=(1/2,1/4)$ and
$(K_1(r_0),K_1(r_1))=(1/4,1/2)$. The known logger uses
\begin{equation}\label{ko:logger}
\begin{aligned}
 \pi_b(c_{k,m})&=\frac1{16}, &
 \pi_b(r_{j,m})&=\frac38 a(m,j) && (h=1),\\
 \pi_b(c_{k,m}\mid M_h=\ell)&=\frac18 a(m,\ell), &
 \pi_b(r_{j,m}\mid M_h=\ell)&=K_\ell(r_j)a(m,j) && (2\le h\le H).
\end{aligned}
\end{equation}
The second line also specifies the logged terminal action, without a
subsequent state or memory update. The logger ignores observation values
and never accesses the latent state. Its total preserve probability is
always $1/4$, and its smallest action probability is $1/32$, giving
$C_A=32$. The memoryless target chooses $c_{1,0}$ and $c_{1,1}$ with
probability $1/2$ each at every preterminal step and is uniform over
actions at step $H$. Consequently,
\begin{equation}\label{ko:values}
 J_\omega=\frac12(p_{\omega,1,0}+p_{\omega,1,1})
 =\frac12+\omega,\qquad (J_-,J_+)=(1/4,3/4).
\end{equation}
All seed probabilities belong to $\{1/8,3/8,5/8,7/8\}$, so the construction is a valid, exactly realizable two-model class.

\subsection{Shared operators and history coverage}\label{ko:shared}

Let $\tau_{h-1}=(O_1,A_1,\ldots,O_{h-1},A_{h-1})$ and let
$b_h(\tau_{h-1})$ be the posterior distribution of $S_h$.
The action-containing future is
$F_h=(O_h,A_h,\ldots,O_{H-1},A_{H-1},O_H)$, and the observation-only
future is $G_h=(O_h,\ldots,O_H)$.
Abbreviate their conditional probability matrices as $U_h=U_{F,h}$ and
$U_h^{\mathrm{obs}}=U_{F,h}^{\mathrm{obs}}$, with fixed row and state
labels across models.

For each writer $m$, averaging~\eqref{ko:seed} over the equally likely
commands $k$ gives $\bar p_m=1/2+(2m-1)/8$. Hence
\begin{equation}\label{ko:joint}
\begin{aligned}
 \Pp_b^\omega(S_2=\mathsf{H}_b,M_2=m)&=\frac18\Ber_{\bar p_m}(b),\\
 \Pp_b^\omega(S_2=\mathsf{R}_b,M_2=m)&=\frac38\sum_j a(m,j)\Ber_{q_j}(b),
\end{aligned}
\end{equation}
where $\Ber_p(b)=p^b(1-p)^{1-b}$. Both lines are independent of $\omega$.
From step $2$ onward, the state--memory process and its observation kernel
are common. Its marginal distribution is therefore common at every step,
as is the suffix law conditional on any $(S_h,M_h)$.
Mixing that law against the common $\Pp_b(M_h\mid S_h)$ proves that
$U_h$ and $U_h^{\mathrm{obs}}$ are common for $h\ge2$.
At $h=1$, $U_1^{\mathrm{obs}}$ is also common: deleting $A_1$ leaves the
fixed initial observation followed by the common process initialized
by~\eqref{ko:joint}. The action-containing $U_1$ retains the command $k$
and need not be common.

The history Grams $\Sigma_{H,h}=\E_b[b_hb_h^\top]$ are common as well.
Exchange $c_{0,m}$ and $c_{1,m}$ only in the first action of a history,
leaving all other tokens fixed. This is a probability-preserving bijection
between the two models because $p_{-\omega,1-k,m}=p_{\omega,k,m}$.
It preserves the writer path and the posterior state vector, hence
preserves the Gram sum in the same state coordinates.

For the coverage bound, put $v(p)=(1-p,p)^\top$. The last action identifies the current hidden/reset type, so the history Gram is block diagonal. At $h=2$, its hidden block is
\begin{equation}\label{ko:history2}
 \frac1{16}\sum_{p\in\{1/8,3/8,5/8,7/8\}}v(p)v(p)^\top
 =\frac1{256}\begin{pmatrix}21&11\\11&21\end{pmatrix},
 \qquad \lambda_{\min}=\frac5{128}.
\end{equation}
At every $h\ge2$, the reset block equals
$\frac38\sum_{j=0}^1v(q_j)v(q_j)^\top$ and has minimum eigenvalue $6/25$. For $h\ge3$, retain just histories with a reset $r_j$ at $h-2$ followed by any preserve action. Each such coarse event has probability $3/32$ and conditional current bit law $\Ber(q_j)$. Conditioning on the full history can only increase the posterior second moment in positive-semidefinite order. The hidden block is therefore at least
$\frac3{32}\sum_jv(q_j)v(q_j)^\top$, whose minimum eigenvalue is $3/50>5/128$. Since $\Sigma_{H,1}=[1]$, this proves the exact value $C_H=128/5$.

\subsection{Uniform revealing bounds}\label{ko:revealing}

Memory is uniform from step $2$ onward: its transition probabilities under~\eqref{ko:logger} are $5/8$ to keep its value and $3/8$ to switch. The marginal probability of each reset command is consequently $3/8$. Let
$d_t=\Pp_b(B_t=1\mid M_t=1)-\Pp_b(B_t=1\mid M_t=0)$, where $B_t$ is the physical bit. From~\eqref{ko:joint} and the subsequent kernels,
\begin{equation}\label{ko:anchor}
 d_2=\frac{29}{80},\qquad d_{t+1}=\frac18d_t+\frac3{10},\qquad
 d_t=\frac{12}{35}+\frac{11}{560}\,8^{-(t-2)}.
\end{equation}
Indeed, writing $x_{t,m}=\Pp_b(B_t=1\mid M_t=m)$, the two updates are
$x_{t+1,0}=(3x_{t,0}+x_{t,1})/16+9/40$ and
$x_{t+1,1}=(x_{t,0}+3x_{t,1})/16+21/40$; their sum stays one.
Both types have a uniform bit marginal, giving
$p_h=(1/8,1/8,3/8,3/8)$ in the order $(\mathsf{H}_0,\mathsf{H}_1,\mathsf{R}_0,\mathsf{R}_1)$ for all $h\ge2$.
The hidden bit--memory contrast is
\begin{equation}\label{ko:eta}
 \eta_h:=\Pp_b(M_h=1\mid \mathsf{H}_1)-\Pp_b(M_h=1\mid \mathsf{H}_0),
 \qquad \eta_2=\frac14,\quad \eta_h=\frac12d_{h-1}\ge\frac6{35}\quad(h\ge3).
\end{equation}
For $h\ge3$, the factor $1/2$ is the contraction of the writer channel $a$ after a preserve action; Bayes' rule gives the reverse conditional contrast because both hidden bit and memory marginals are uniform.

For any two distinct probability columns $P_0,P_1$, their Gram
$\Sigma=P^\top\operatorname{diag}(P\mathbf1)^{-1}P$ has eigenvalues $1$ and
\begin{equation}\label{ko:binary}
 \lambda=\frac12\sum_y\frac{(P_0(y)-P_1(y))^2}{P_0(y)+P_1(y)}
 \ge\TV(P_0,P_1)^2,
 \qquad \|\Sigma^{-1}\|_1=\lambda^{-1}.
\end{equation}
The inequality follows from Cauchy--Schwarz; the inverse identity follows by diagonalizing along $(1,1)$ and $(1,-1)$.
For $2\le h\le H-2$, reset commands $r_0,r_1$ produce next-observation Bernoulli parameters $9/50,41/50$. Their gap is $16/25$, and the two memories change the reset-command probability by $1/4$. Thus
\begin{equation}\label{ko:tv}
 \TV\bigl(\mathcal L(O_{h+1}\mid \mathsf{H}_0),\mathcal L(O_{h+1}\mid \mathsf{H}_1)\bigr)
 =\frac{16}{25}\frac14\eta_h\ge\frac{24}{875}.
\end{equation}
At $h=H-1$, preserve directly reveals the old bit at the terminal step, while reset contributes in the same direction; the corresponding TV is $1/4+\eta_h/5\ge1/4$. The reset columns already have TV $4/5$ from the current observation. That observation also separates hidden and reset supports, so each future Gram consists of two two-column blocks. Keeping the whole future cannot decrease TV. Equation~\eqref{ko:binary} therefore gives
\begin{equation}\label{ko:coefficients}
 C_F^{\mathrm{obs}}\le(875/24)^2=\frac{765625}{576}<1330.
\end{equation}
The only remaining cut, $h=1$, has a single probability column and Gram $[1]$.

Let $C_F$ denote the analogous coefficient for the action-containing operator, and write $D_{p_h}=\operatorname{diag}(p_h)$. Restoring action tokens refines each two-column experiment. Conditional Jensen implies that its Gram increases in positive-semidefinite order. The eigenvector $(1,1)$ still has eigenvalue one, so~\eqref{ko:binary} gives the same bound for $C_F$. Finally, within each type block the two state-prior weights are equal. Writing their common weight as $w$, the prior-weighted Gram satisfies
\begin{equation}\label{ko:weighted}
 D_{p_h}U_h^\top\operatorname{diag}(U_hp_h)^{-1}U_h
 =U_h^\top\operatorname{diag}(U_h\mathbf1)^{-1}U_h,
\end{equation}
because the blockwise factors $w$ cancel. The current observation separates
the blocks, and the equality also holds at $h=1$. Its inverse-$1$-norm
coefficient therefore satisfies $C_F^p=C_F\le765625/576<1330$.

\subsection{The complete-log lower bound}\label{ko:lower}

Let $E_H$ be the observable event that all $H-1$ preterminal actions preserve, and put $q_H=4^{-(H-1)}$. Its probability is $q_H$ in both models. On $E_H^c$, the first reset occurs before any observation reveals the seed: all preceding seed-dependent states emit blanks, and the logger ignores their bit. Reset erases the seed, so the two complete-log laws agree on every atom outside $E_H$.
On $E_H$, the first $(k,m)$ is uniform over four possibilities.
The remaining action/writer path has a common law, and the terminal bit
is $\Ber(p_{\omega,k,m})$. The logged terminal action uses the same
conditional kernel in both models given the recorded memory $M_H$,
so it contributes zero conditional KL. This behavior kernel differs
from the target's uniform terminal kernel. Under either policy, $A_H$
follows the reward and leaves its value unchanged.
Rewards are fixed by observations.
Thus, with $\kappa_{\mathrm{com}}=\frac14\log(35/3)$,
\begin{equation}\label{ko:kl}
 \KL(P_b^-\|P_b^+)
 =\frac{q_H}{4}\sum_{k,m}\operatorname{kl}(p_{-,k,m}\|p_{+,k,m})
 =\kappa_{\mathrm{com}}q_H,
 \qquad \TV(P_b^-,P_b^+)=\frac{q_H}{2}.
\end{equation}
For the KL identity, each of the unordered pairs $(1/8,5/8)$ and $(3/8,7/8)$ contributes
$\operatorname{kl}(u\|v)+\operatorname{kl}(v\|u)=\frac12\log(35/3)$.
Exchanging models permutes the four summands, so the reverse KL is identical. Each labeled Bernoulli pair has TV $1/2$, yielding the TV identity.

Providing the exact common matrices $(U_h^{\mathrm{obs}})_{h<H}$ adds the same deterministic input to both experiments and leaves these divergences unchanged. The same holds for their common left inverses, for example
\begin{equation}\label{ko:leftinverse}
 D_h=(\Sigma_{F,h}^{\mathrm{obs}})^{-1}(U_h^{\mathrm{obs}})^\top
 \operatorname{diag}(U_h^{\mathrm{obs}}\mathbf1)^{-1},
 \qquad D_hU_h^{\mathrm{obs}}=I,
\end{equation}
with zero-probability rows omitted.
For $n\le1/(8\kappa_{\mathrm{com}}q_H)$, the product-law KL is at most
$1/8$, hence its TV is at most $1/4$ by Pinsker's inequality.
Thresholding any value estimate at $1/2$ gives a test that is correct
whenever its estimation error is at most $1/8$.
The sum of the two testing errors is at least $3/4$, so
\begin{equation}\label{ko:lecam}
 \inf_{\widehat J}\max_{\omega\in\{-1/4,1/4\}}
 \Pp_{(P_b^\omega)^{\otimes n}}
 \left(\left|\widehat J-J_\omega\right|>\frac18\right)\ge\frac38.
\end{equation}
This completes the proof of Proposition~\ref{prop:known-operators}.

\section{Proofs for pre-overwrite readout}
\label{app:readout}

We prove Theorem~\ref{thm:readout} and
Corollary~\ref{cor:readout-uniform} for the four-action family in
\cref{eq:continuum}.
Throughout, $\theta_0=1/20$, $\vartheta=\theta-\theta_0$,
$a_0=1/4$, and $a_1=3/4$. Thus
$|\vartheta|\le1/20$, $p_{\theta,k}=a_k+\vartheta$, and
$J_\theta=1/2+\vartheta$.

\subsection{Protocol and full-record likelihood}
\label{app:readout-likelihood}

At the first step, the logger chooses each preserve action $c_k$ with
probability $1/8$. This creates a hidden state with bit
$B\sim\operatorname{Ber}(a_k+\vartheta)$ and sets public memory to $k$.
The two reset actions each have probability $3/8$.
A reset $r_j$ creates a reset state with a fresh bit of success probability
$q_0=1/10$ or $q_1=9/10$, and sets memory to $j$.
At later preterminal steps, each preserve action still has probability
$1/8$, keeps both the bit and memory, and returns to the hidden type.
The reset probabilities are $(1/2,1/4)$ at memory $0$ and
$(1/4,1/2)$ at memory $1$.
Hidden states emit a blank. Reset states with bit $0$ or $1$ emit
$\operatorname{Ber}(1/10)$ or $\operatorname{Ber}(9/10)$, respectively.
The terminal observation reveals the current bit exactly, and the
terminal reward equals that bit. All earlier rewards vanish.
The target chooses the two preserve actions uniformly and never resets.
Both policies use the same uniform terminal action kernel, after the
terminal observation and reward.

Write $K$ for the first preserve label. The original bit is still present
before action $A_h$ precisely when $A_1,\ldots,A_{h-1}$ all preserved.
At such a history, $M_h=K$, and the preterminal observations are blank.
Conditional on $K$, this public route is independent of $B$.
The distribution of $K$ itself does not depend on $\theta$.
A read may occur after the original observation and before $A_h$,
for $2\le h\le H-1$, provided no earlier read or reset has occurred.
Its selected reliability $\xi\in[-1,1]$ and cost are known and recorded,
and its output satisfies
\begin{equation}
 \mathbb P(Y=1\mid B,\xi)=\frac12+\xi(B-1/2).
 \label{eq:readout-proof-channel}
\end{equation}
The read uses fresh randomness independent of all original transition,
emission, and action randomness and of the controller's randomization.
In particular, conditional on the original bit, selected reliability,
and original trajectory, its law remains
\eqref{eq:readout-proof-channel}.
This independence prevents a later reset from revealing the read noise.
The read changes no physical state or original policy. No further
measurement is taken, but the full original log continues to the terminal
step.

Let $\mathcal Q$ be the event that a read occurs, and set
\begin{equation}
 E_H=\{A_1,\ldots,A_{H-1}\text{ all preserve}\},\qquad
 q_H=4^{-(H-1)},\qquad s_h=4^{-(H-h)}.
 \label{eq:readout-proof-survival}
\end{equation}
The probability of $E_H$ is $q_H$. From any eligible history at step $h$,
the probability of preserving through the terminal step is $s_h$.
These probabilities do not depend on the bit or the read.

Fix a causal read policy $\pi$, and let $P_\theta^\pi$ denote its full
record law, including the original log and any read.
There are three disjoint cases.
On $E_H$, the terminal observation is the original $B$; conditional on
$B$, an earlier read contributes no further parameter dependence.
On $\mathcal Q\cap E_H^c$, the read precedes the first reset and is the
only observation of the original bit. After that reset, the original
suffix is independent of $B$, conditional on the recorded route.
On $\mathcal Q^c\cap E_H^c$, no record of the original bit survives.
This last case includes a first-step reset.

Define
\[
 A_{k,\xi}=\frac12+\xi(a_k-1/2).
\]
On a read branch, $Y$ therefore has success probability
$A_{k,\xi}+\xi\vartheta$.
The three cases give the following observable score for the full record:
\begin{equation}
 \ell=
 \begin{cases}
 \displaystyle\frac{B-a_K}{a_K(1-a_K)},& E_H,\\[6pt]
 \displaystyle\frac{\xi(Y-A_{K,\xi})}
 {A_{K,\xi}(1-A_{K,\xi})},&\mathcal Q\cap E_H^c,\\[6pt]
 0,&\mathcal Q^c\cap E_H^c.
 \end{cases}
 \label{eq:readout-proof-score}
\end{equation}
Here $B$ in the first line is the terminal observation.
Every remaining factor in the full-record probability is known and
parameter independent, so
\begin{equation}
 \frac{dP_\theta^\pi}{dP_{\theta_0}^\pi}=1+\vartheta\ell,
 \qquad \mathbb E_{\theta_0}^\pi\ell=0,
 \qquad |\ell|\le4,
 \qquad \frac45\le1+\vartheta\ell\le\frac65 .
 \label{eq:readout-proof-affine}
\end{equation}
For an exact observation the score bound follows from
$\min_k\{a_k,1-a_k\}=1/4$.
For a noisy read, its absolute value is at most
$|\xi|/(1/2-|\xi|/4)\le4$.
Thus all bounds in \eqref{eq:readout-proof-affine} hold on the entire
parameter interval, including its endpoints.

\subsection{Information and feasible read policies}
\label{app:readout-flow}

The center information of one noisy read is independent of $k$:
\begin{equation}
 \psi(\xi)=\frac{\xi^2}{A_{k,\xi}(1-A_{k,\xi})}
          =\frac{16\xi^2}{4-\xi^2},\qquad
 \psi(1)=\frac{16}{3}.
 \label{eq:readout-proof-psi}
\end{equation}
Let $\mathcal F$ be the finite tree of eligible public histories before
any read. Its two roots at step $2$ have masses $\nu_f=1/8$ each;
all other $\nu_f$ are zero.
At node $f$, let $d_f$ be the probability of reaching it, $w_f$ the
probability of waiting there, and $x_{f,\alpha}$ the probability of selecting
read option $\alpha$ from its finite menu $\mathcal A_f$.
Let $T(f\mid g)$ be the known probability that a wait at $g$ reaches
eligible child $f$.
Each preserve child has probability $1/8$; reset branches exit this tree.
There are no eligible children after step $H-1$.
The flow constraints are
\begin{equation}
 d_f=\nu_f+\sum_{g\in\mathcal F}w_gT(f\mid g),
 \qquad w_f+\sum_{\alpha\in\mathcal A_f}x_{f,\alpha}=d_f,
 \qquad d_f,w_f,x_{f,\alpha}\ge0.
 \label{eq:readout-proof-flow}
\end{equation}
The tree ends when a read is selected because no further measurement
decision remains. The physical trajectory does not end there.

Every causal randomized policy induces a flow
\eqref{eq:readout-proof-flow}.
Conversely, any such flow is realized by choosing option $\alpha$ with
probability $x_{f,\alpha}/d_f$ and waiting with probability $w_f/d_f$ at
each node of positive mass.
Choices at zero-mass nodes are arbitrary.
Controller randomization can be averaged conditional on each node,
since no bit-dependent observation has yet been read.
These flows and their expected costs are independent of $\theta$.
The history tree is finite; it need not have polynomial size in $H$.

Set $\mathcal I(\pi)=\mathbb E_{\theta_0}^\pi\ell^2$.
The event $E_H$ contributes $(16/3)q_H$, regardless of whether a read
was taken. A read selected at history $f$ of depth $h(f)$ contributes
only when a later reset occurs. Conditional on this read decision,
that event has probability $1-s_{h(f)}$ and is independent of its
output given $K$.
Consequently,
\begin{equation}
 \begin{aligned}
 \mathcal I(\pi)&=\frac{16}{3}q_H+
     \sum_{f,\alpha}x_{f,\alpha}(1-s_{h(f)})\psi(\xi_\alpha),\\
 C(\pi)&=\sum_{f,\alpha}x_{f,\alpha}c_\alpha .
 \end{aligned}
 \label{eq:readout-proof-information}
\end{equation}
The factor $1-s_{h(f)}$ removes trajectories whose terminal observation
already reveals the same bit. It is necessary even when the read is
exact.

Let $\mathcal X_b$ be the flows satisfying
\eqref{eq:readout-proof-flow} and $C(\pi)\le b$.
Maximizing \eqref{eq:readout-proof-information} gives the exact information
in Theorem~\ref{thm:readout}:
\begin{equation}
 \mathcal I_\star(b)=\frac{16}{3}q_H+
 \max_{x\in\mathcal X_b}
 \sum_{f,\alpha}x_{f,\alpha}(1-s_{h(f)})
 \frac{16\xi_\alpha^2}{4-\xi_\alpha^2}.
 \label{eq:readout-information}
\end{equation}
This is a finite linear program with known coefficients.
The feasible set is nonempty and compact: never reading is feasible,
and each variable is a probability mass.
An optimal flow therefore exists and yields an executable fixed policy.
In particular, $\mathcal I_\star(b)\ge(16/3)q_H>0$.

\subsection{Estimator and uniform upper bound}
\label{app:readout-upper}

Equation~\eqref{eq:readout-proof-affine} gives, for every fixed policy,
\begin{equation}
 \mathbb E_\theta^\pi\ell=\vartheta\mathcal I(\pi),\qquad
 \operatorname{Var}_\theta^\pi(\ell)
 \le\mathbb E_\theta^\pi\ell^2
 \le\frac65\mathcal I(\pi).
 \label{eq:readout-proof-moments}
\end{equation}
Independently execute an optimal flow policy on $n$ trajectories.
Write $\mathcal I_\star=\mathcal I_\star(b)$ and use
\begin{equation}
 \widehat J=
 \operatorname{clip}_{[9/20,\,11/20]}
 \left(\frac12+\frac1{n\mathcal I_\star}\sum_{i=1}^n\ell_i\right).
 \label{eq:readout-proof-estimator}
\end{equation}
Before clipping, this estimator is unbiased for $J_\theta$.
Clipping cannot increase its distance to any value in the parameter
range. Chebyshev's inequality therefore yields
\begin{equation}
 \begin{aligned}
 \sup_{\theta\in[0,1/10]}
 P_\theta\{|\widehat J-J_\theta|>\varepsilon\}
 &\le\frac{6}{5n\mathcal I_\star\varepsilon^2}\\
 &\le\frac14
 \quad\text{if}\quad
 n\ge\frac{24}{5\mathcal I_\star\varepsilon^2}.
 \end{aligned}
 \label{eq:readout-proof-upper}
\end{equation}
\subsection{Lower bound for adaptive read policies}
\label{app:readout-lower}

Now allow an algorithm to choose the next trajectory's entire causal
read rule using all previous records. Its original trajectories remain
independent draws from the unchanged POMDP.
Suppose its total expected cost is at most $nb$ for every $\theta$.
Take $0<\varepsilon\le1/40$ and
\[
 \theta_-=\theta_0-2\varepsilon,\qquad
 \theta_+=\theta_0+2\varepsilon .
\]
Both parameters are admissible, and their target values differ by
$4\varepsilon$.

Conditional on the previous transcript, denote the next rule by $\pi_i$.
The same conditional rule applies under either parameter at that
transcript. Its record laws satisfy
\eqref{eq:readout-proof-affine}.
Using $D_{\rm KL}(P\Vert Q)\le\chi^2(P\Vert Q)$ and
$1+2\varepsilon\ell_i\ge4/5$ gives
\begin{equation}
 \begin{aligned}
 D_{\rm KL}(P_{\theta_-}^{\pi_i}\Vert
            P_{\theta_+}^{\pi_i})
 &\le16\varepsilon^2
   \mathbb E_{\theta_0}^{\pi_i}
   \left[\frac{\ell_i^2}{1+2\varepsilon\ell_i}\right]\\
 &\le20\varepsilon^2\mathcal I(\pi_i).
 \end{aligned}
 \label{eq:readout-proof-conditional-kl}
\end{equation}
If the algorithm uses additional randomization, include its choices in
the transcript. Their conditional kernels are parameter independent
and contribute zero KL.

Average the conditional flows of $\pi_1,\ldots,\pi_n$ under the
$\theta_-$ transcript law.
Each flow satisfies the same linear conservation constraints.
For every fixed conditional rule, its pre-read flow and expected cost
are independent of $\theta$.
Thus the average cost equals the algorithm's expected cost under
$\theta_-$ divided by $n$, and is at most $b$.
The averaged flow is feasible for the definition of $\mathcal I_\star$, which
implies
\begin{equation}
 \mathbb E_{\theta_-}\sum_{i=1}^n\mathcal I(\pi_i)\le n\mathcal I_\star.
 \label{eq:readout-proof-average-flow}
\end{equation}
Applying the KL chain rule to the full adaptive experiment,
\begin{equation}
 D_{\rm KL}(P_{\theta_-}^{\mathcal A}\Vert
            P_{\theta_+}^{\mathcal A})
 \le20n\varepsilon^2\mathcal I_\star.
 \label{eq:readout-proof-adaptive-kl}
\end{equation}
For $n\le1/(160\mathcal I_\star\varepsilon^2)$ this is at most $1/8$.
Pinsker's inequality then bounds total variation by $1/4$.
Thresholding any value estimate at the midpoint of the two target
values gives a test whose error is no larger than its estimation
failure probability at accuracy $\varepsilon$.
The two-point testing bound therefore gives
\[
 \max_{\theta\in\{\theta_-,\theta_+\}}
 P_\theta^{\mathcal A}
 \{|\widehat J-J_\theta|>\varepsilon\}\ge\frac38 .
\]
Together with \eqref{eq:readout-proof-upper}, this proves
Theorem~\ref{thm:readout}, including the minimax rate at confidence $2/3$.
The budget throughout is an expected cost, not a pathwise cap.

\subsection{Equal-quality reads with unit cost}
\label{app:readout-uniform}

\begin{corollary}[Uniform readout quality and cost]
\label{cor:readout-uniform}
If the only read option at every eligible history has reliability $\xi$ and
unit cost, then
\begin{equation}
 \mathcal I_\star(b)=\frac{16}{3}
 \left[q_H+\frac{3\xi^2}{4-\xi^2}
 \left(1-4^{-(H-2)}\right)\min\{b,1/4\}\right].
 \label{eq:readout-uniform}
\end{equation}
An optimal rule reads at step $2$ with probability $\min\{4b,1\}$ when the
first action preserves, and never reads otherwise.
\end{corollary}

\begin{proof}
Let $y_h$ be the unconditional
probability of reading at step $h$.
Only the first-preserve branch is eligible, and its probability is
$1/4$. Hence $\sum_h y_h\le\min\{b,1/4\}$.
Since $s_h\ge s_2=4^{-(H-2)}$,
\begin{equation}
 \mathcal I(\pi)-\frac{16}{3}q_H
 \le\psi(\xi)(1-4^{-(H-2)})\min\{b,1/4\}.
 \label{eq:readout-proof-earliest}
\end{equation}
Read at step $2$ with probability $\min\{4b,1\}$ conditional on a first
preserve action, and otherwise never read.
This policy attains equality, giving
\begin{equation}
 \mathcal I_\star(b)=\frac{16}{3}
 \left[q_H+\frac{3\xi^2}{4-\xi^2}
       (1-4^{-(H-2)})\min\{b,1/4\}\right].
 \label{eq:readout-proof-uniform}
\end{equation}
At $\xi=0$ the read adds no information.
At $|\xi|=1$ and $b=1/4$, the expression equals $4/3$: every bit created
by a first preserve action is observed, with no double counting of its
terminal readout.
\end{proof}

Equation~\eqref{eq:readout-total} follows with $b=A/n$.

\subsection{Moving a read earlier preserves every decision risk}
\label{app:readout-dominance}

The earliest-read rule is optimal for a stronger reason than its center
information. With a fixed channel, a delayed experiment can be reconstructed
from an early one without knowing the parameter.

\begin{lemma}[Early-readout dominance]
\label{lem:early-readout}
Restrict the protocol to one known channel of reliability $\xi$ and unit
cost at every eligible history. For each causal single-read rule $\pi$,
there is a rule $\pi^\uparrow$ that reads only at step $2$, with the same
conditional read probability given each $A_1=c_k$. Its expected cost equals
that of $\pi$ for every $\theta$. There is a parameter-independent randomized
kernel $\mathsf K_\pi$ on complete records such that
\begin{equation}
 P_\theta^\pi=P_\theta^{\pi^\uparrow}\mathsf K_\pi
 \qquad\text{for every }\theta\in[0,1/10].
 \label{eq:readout-dominance}
\end{equation}
Thus every decision rule based on $\pi$ can be implemented using
$\pi^\uparrow$ with exactly the same risk at every parameter.
\end{lemma}

\begin{proof}
Let $T$ be the read time, with $T=\infty$ if no read occurs. Conditional on
$A_1=c_k$, define
\[
 e=4^{-(H-2)},\qquad
 r_k=P_\theta^\pi(\mathcal Q\mid A_1=c_k),\qquad
 t_k=P_\theta^\pi(\mathcal Q\cap E_H^c\mid A_1=c_k).
\]
These quantities do not depend on $\theta$: before the only read, the
eligible public route contains preserve commands and blank observations,
independent of the seed bit. Since the logger ignores the read,
\begin{equation}
 t_k=\sum_{h=2}^{H-1}P_\theta^\pi(T=h\mid A_1=c_k)(1-s_h)
 \le r_k(1-e).
 \label{eq:readout-noisy-mass}
\end{equation}

Retain $A_1$ and reduce a complete record to three branches: on $E_H$ keep
the exact terminal bit $B$; on $\mathcal Q\cap E_H^c$ keep the noisy read
$Y$; otherwise keep a null symbol. Given $A_1=c_k$, the reduced experiment is
\begin{equation}
\begin{array}{c|c|c}
 \text{branch}&\text{probability}&\text{observation law}\\ \hline
 \text{exact}&e&\Ber(p_{\theta,k})\\
 \text{noisy}&t_k&\Ber(1/2+\xi(p_{\theta,k}-1/2))\\
 \text{erased}&1-e-t_k&\text{null}.
\end{array}
\label{eq:readout-three-branches}
\end{equation}
A first-step reset gives a parameter-independent null branch.

This reduction loses no parameter information. Indeed, the score in
\eqref{eq:readout-proof-score} depends only on the retained statistic $Z$.
For a full record $z$ mapped to $Z$, \eqref{eq:readout-proof-affine} gives
\[
 P_\theta^\pi(z\mid Z)=P_{\theta_0}^\pi(z\mid Z).
\]
Hence the reference conditional law defines a parameter-independent reverse
kernel $\mathsf R_\pi$ from $Z$ to the full record. It reconstructs all
original observations, rewards, and actions, including the terminal action,
as well as read metadata. On the exact branch, any earlier read is generated
from the known channel conditional on $B$. On the noisy branch, the original
route and reset suffix are independent of the seed given $K$ and the branch.
The reconstructed route need not equal the input route. Zero-mass branches
may be assigned arbitrary kernels.

Now let $\pi^\uparrow$ read at step $2$ with probability $r_k$ given
$A_1=c_k$, independently of the original process. It never reads otherwise.
The conditional expected cost remains $r_k$. Its reduced experiment has
exact mass $e$ and noisy mass $t_k^\uparrow=r_k(1-e)\ge t_k$.
Keep every exact observation. Retain a noisy observation with probability
$t_k/t_k^\uparrow$, independently of its value, and replace it by null
otherwise. If $t_k^\uparrow=0$, both noisy masses vanish. This thinning
recovers exactly \eqref{eq:readout-three-branches} for $\pi$.
Composing reduction, thinning, and $\mathsf R_\pi$ proves
\eqref{eq:readout-dominance}. Composing this kernel with any decision rule
proves the risk statement. Both policies have expected cost
$(r_0+r_1)/8$ under every parameter.
\end{proof}

The comparison also applies to adaptation across trajectories under a total
expected budget. Maintain a simulated transcript for the original algorithm.
Given this transcript, implement the early version of its next conditional
read rule, then feed the reconstructed record back to the algorithm.
Induction gives the original transcript law for every $\theta$; conditional
expected costs also agree, so their sums agree. Realized costs need not match.
The early probabilities $r_0,r_1$ depend on the original rule and may differ;
the lemma does not require the symmetric allocation in
Corollary~\ref{cor:readout-uniform}.

\section{Target evidence with an unknown readout bias}
\label{app:readout-calibration}

A read can retain seed-dependent information without separating the seed
mean from a sensor bias. A read of a reset-generated bit supplies the
missing calibration. We prove both statements for the complete logged
experiment, then derive the budget law in
Theorem~\ref{thm:readout-calibration}.

\subsection{Model and readout protocols}
\label{app:cal-model}

Keep the four-action family in~\eqref{eq:continuum} and all its original
observations, actions, and rewards. Write
\begin{equation}
 \theta\in[0,1/10],\qquad
 d_0=1/5,\quad d_1=7/10,\qquad
 p_{\theta,k}=d_k+\theta,\qquad
 J_\theta=9/20+\theta.
 \label{eq:cal-family}
\end{equation}
The extra sensor has known gain $1/2$ and unknown bias
$\beta\in[-1/20,1/20]$:
\begin{equation}
 C_\beta(1\mid B)=\Pr(Y=1\mid B)=1/4+\beta+B/2.
 \label{eq:cal-channel}
\end{equation}
The same pair $(\theta,\beta)$ governs all trajectories. Conditional on the
measured bit, each read uses fresh noise independent of the original
process and of the read-selection randomization. It changes no original
state, action, observation, or reward. The logger ignores the readout.
The complete record contains the original log and the read decision, time,
and outcome. All sensor probabilities lie in $[1/5,4/5]$.

We use two readout protocols. In the \emph{pre-reset protocol}, at most one
read is allowed after $O_h$ and before $A_h$, for $2\le h<H$, provided all
earlier actions preserve. Its timing can depend on public history and
private randomness. In the \emph{step-$2$ calibration protocol}, at most
one read is allowed after $O_2$ and before $A_2$. It reads the original seed
if $A_1=c_k$, or the fresh reset bit if $A_1=r_j$. These two choices are
called an early read and a calibration read, respectively. In either
protocol, a controller may use previous complete records to select the
next trajectory's rule.

Let $R_E,R_C$ indicate the two read types. A fixed allocation $(u,v)$ uses
read probabilities $4u$ after either preserve label and $4v/3$ after either
reset label, independently of $O_2$. Thus
\begin{equation}
 \Pr(R_E=1)=u,\qquad \Pr(R_C=1)=v,\qquad
 0\le u\le1/4,\quad 0\le v\le3/4.
 \label{eq:cal-allocation}
\end{equation}
The total expected cost per trajectory is $u+v$. For lower bounds, we
also allow step-$2$ selection to depend on the first action label and on
$O_2$. The observation available before selection is blank after preserve
and has a known distribution after reset.

Throughout this appendix, set
\begin{equation}
 E_H=\{A_1,\ldots,A_{H-1}\text{ all preserve}\},\qquad
 q=4^{-(H-1)},\qquad s=4^{-(H-2)}=4q.
 \label{eq:cal-survival}
\end{equation}
Here $q=\Pr(E_H)$, and $s$ is the probability that the step-$2$ bit
survives to the terminal observation. Both probabilities are independent
of the bit and the readout rule. On $E_H$, the original seed is observed
exactly at the terminal step.

\subsection{Pre-reset reads alone retain the exponential rate}
\label{app:cal-early-only}

\begin{proposition}[Unknown bias under pre-reset readout]
\label{prop:unknown-bias}
For $H\ge3$ and $0<\varepsilon\le1/40$, the minimax trajectory complexity
under the pre-reset protocol is
$\Theta(q^{-1}\varepsilon^{-2})$ at success probability $2/3$, uniformly
over~\eqref{eq:cal-family}--\eqref{eq:cal-channel}. The lower bound allows
any causal read timing and adaptation across trajectories, with no
restriction on the number of eligible trajectories read.
\end{proposition}

\begin{proof}
Choose
\begin{equation}
 \theta_\pm=1/20\pm2\varepsilon,\qquad
 \beta_\pm=\mp\varepsilon.
 \label{eq:cal-confounded-pair}
\end{equation}
Both parameter pairs are admissible and their target values are
$1/2\pm2\varepsilon$. Given the first preserve label $K=k$, a pre-reset
read has mean
\[
 1/4+d_k/2+\beta_\pm+\theta_\pm/2
 =1/4+d_k/2+1/40.
\]
The equality extends to all complete records that contain a reset.
To see this, fix a causal rule $\pi$. Before its only read, an eligible
route contains preserve commands and blank observations. Conditional on
$K$, this route and the read decision are independent of the seed.
The first reset draws a fresh bit from the known law associated with its
recorded command. Conditional on the recorded action route, the entire
original suffix after that reset is independent of the seed and of the
earlier sensor noise. Consequently, every record $z$
with a read and a later reset has probability
\begin{equation}
 P_{\theta,\beta}^{\pi}(z)
 =g_\pi(z)\big[(1-p_{\theta,k})C_\beta(y\mid0)
                    +p_{\theta,k}C_\beta(y\mid1)\big],
 \label{eq:cal-reset-factor}
\end{equation}
where $g_\pi(z)$ is parameter independent. A reset record without a read
is itself parameter independent. This includes a first-step reset.
Hence the two restricted complete-record laws agree on $E_H^c$.

On $E_H$, expose the common public route and read decision first.
The terminal bit has law $\Ber(p_{\theta_\pm,K})$. If a read was taken,
the record additionally contains $Y$ with law $C_{\beta_\pm}(\cdot\mid B)$.
For $a,b\in[1/5,4/5]$,
\begin{equation}
 \operatorname{kl}(a\Vert b)
 \le\frac{(a-b)^2}{b(1-b)}\le\frac{25}{4}(a-b)^2.
 \label{eq:cal-bernoulli-kl}
\end{equation}
The seed probabilities differ by $4\varepsilon$ and the conditional
sensor probabilities by $2\varepsilon$. The KL chain rule therefore gives
\begin{equation}
 \KL(P_-^\pi\Vert P_+^\pi)
 \le q\big(100\varepsilon^2+25\varepsilon^2\big)
 =125q\varepsilon^2.
 \label{eq:cal-early-kl}
\end{equation}
The second term retains the calibration information in the rare joint
observation $(B,Y)$.

For an adaptive controller, condition on each previous transcript.
The same conditional rule is used under either hypothesis, and the next
original trajectory is fresh. Applying~\eqref{eq:cal-early-kl} to each
conditional rule bounds the KL of $n$ complete records by
$125nq\varepsilon^2$. Private controller randomization has the same
conditional law under both hypotheses and adds no KL.
If $n\le(1000q\varepsilon^2)^{-1}$, Pinsker's inequality gives total
variation at most $1/4$. Thresholding any value estimator at $1/2$ and
applying Le Cam's bound yields
\begin{equation}
 \max_{\sigma\in\{-,+\}}
 \Pr_\sigma\bigl(|\widehat J-J_{\theta_\sigma}|>\varepsilon\bigr)
 \ge3/8.
 \label{eq:cal-testing-bound}
\end{equation}

For the upper bound, ignore the sensor and use
$X=q^{-1}\one_{E_H}R_H$. The first preserve label is uniform on $E_H$,
so $\E X=J_\theta$, $0\le X\le q^{-1}$, and
$\operatorname{Var}(X)\le q^{-1}$. The original trajectories remain
independent under adaptive readout. Bernstein's inequality gives
\[
 \Pr\left(\left|n^{-1}\sum_{i=1}^nX_i-J_\theta\right|>
 \varepsilon\right)
 \le2\exp(-3nq\varepsilon^2/8).
\]
Thus $n\ge8\log(6)/(3q\varepsilon^2)$ suffices, with rounding up to
an integer.
\end{proof}

The complete experiment identifies the value for every finite $H$ through
$E_H$. The proposition says that pre-reset reads alone do not change its
exponential sample requirement when the bias is unknown. The one-read
and fixed-gain assumptions specify the experiment: repeated measurements
of the same bit or measurements at different gains can carry additional
joint information.

\subsection{Complete-record information under calibration}
\label{app:cal-full-kl}

Fix any step-$2$ selection rule and let $u,v$ be its unconditional early
and calibration read probabilities. Let $P_-^{\pi},P_+^{\pi}$ denote its
complete-record laws for two parameter pairs, and write
$\Delta\theta=\theta_+-\theta_-$ and
$\Delta\beta=\beta_+-\beta_-$. There are three informative branches.

\paragraph{The original bit survives.}
On $E_H$, the terminal observation reveals $B\sim\Ber(d_K+\theta)$.
If an early read occurred, its outcome has conditional law
$C_\beta(\cdot\mid B)$. The probabilities of these events are $q$ and
$su=4qu$, respectively. The latter identity holds even for label-specific
selection: the remaining preservation probability is $s$ after either
first preserve label.

\paragraph{An early read is followed by a reset.}
This branch has probability $(1-s)u$. Its complete-record dependence on
the parameters is the mixture in~\eqref{eq:cal-reset-factor}. Given $K$,
the read has success probability
$1/4+d_K/2+\beta+\theta/2$. The remaining recorded suffix contributes
only a parameter-independent factor.

\paragraph{The first action resets and a calibration read occurs.}
Given $A_1=r_j$, the measured bit has known prior $\Ber(q_j)$, with
$(q_0,q_1)=(1/10,9/10)$. Let $Z_0$ be the complete original log. Its law
is known on this branch. Conditional on $Z_0$ and on selection, the
posterior
$w(Z_0)=\Pr(B'=1\mid Z_0,A_1=r_j,\text{selected})$ is also known.
The exact conditional readout law is
\begin{equation}
 Y\mid Z_0,\text{selected}
 \sim\Ber\bigl(1/4+\beta+w(Z_0)/2\bigr).
 \label{eq:cal-posterior-channel}
\end{equation}
This formula retains the correlation between $O_2$ and $Y$, and any
later observation of the same reset bit. It remains valid when selection
uses $O_2$. Alternatively, adding $B'$ to the record bounds its KL by
the conditional sensor KL; removing $B'$ can only reduce that bound.
This branch has probability $v$ and depends on $\beta$, not on $\theta$.

All other branches carry neither parameter. Combining these cases with
\eqref{eq:cal-bernoulli-kl} gives the uniform bound
\begin{equation}
 \begin{aligned}
 \KL(P_-^\pi\Vert P_+^\pi)\le\frac{25}{4}\big[&
 q(\Delta\theta)^2+4qu(\Delta\beta)^2\\
 &+u(1-4q)(\Delta\beta+\Delta\theta/2)^2
 +v(\Delta\beta)^2\big].
 \end{aligned}
 \label{eq:cal-full-kl}
\end{equation}
The $4qu$ term is the calibration supplied by an early read whose bit
also reaches the terminal observation. The bound accounts for this
information as well as for reset-generated calibration.

\subsection{Lower bound with adaptive category budgets}
\label{app:cal-budget-proof}

We prove the lower bound in Theorem~\ref{thm:readout-calibration}.
Define
\begin{equation}
 h_{u,v}=\begin{cases}uv/(u+v),&u+v>0,\\0,&u=v=0,\end{cases}
 \qquad Q_{u,v}=q+h_{u,v}.
 \label{eq:cal-information-scale}
\end{equation}
Take $\alpha=u/(u+v)$ when $u+v>0$, and $\alpha=0$ otherwise.
Use the fixed hypotheses
\begin{equation}
 \theta_\pm=1/20\pm2\varepsilon,\qquad
 \beta_\pm=\mp\alpha\varepsilon.
 \label{eq:cal-budget-pair}
\end{equation}
They belong to the parameter rectangle and have value gap $4\varepsilon$.
Their differences satisfy
$\Delta\beta=-\alpha\Delta\theta/2$ and
\begin{equation}
 u(1-\alpha)^2+v\alpha^2=h_{u,v}.
 \label{eq:cal-complement-identity}
\end{equation}
In~\eqref{eq:cal-full-kl}, the two terms on surviving records are at
most $(5/4)q(\Delta\theta)^2$, since $4u\le1$ and
$|\Delta\beta|\le|\Delta\theta|/2$.
The other two terms are at most
$h_{u,v}(\Delta\theta)^2/4$. Hence a fixed rule satisfies
\begin{equation}
 \KL(P_-^\pi\Vert P_+^\pi)
 \le(125q+25h_{u,v})\varepsilon^2
 \le125Q_{u,v}\varepsilon^2.
 \label{eq:cal-pair-kl}
\end{equation}

Now let trajectory $i$'s rule depend on the previous complete transcript
$T_{i-1}$. Denote its conditional read probabilities by $u_i,v_i$.
Given that transcript, the rule and the original observations available
before selection have the same law under both hypotheses. Thus
\eqref{eq:cal-full-kl} applies conditionally, including when selection
depends on $O_2$. Suppose the category budgets obey
\begin{equation}
 \E_{\theta,\beta}\sum_{i=1}^nR_{E,i}\le nu,\qquad
 \E_{\theta,\beta}\sum_{i=1}^nR_{C,i}\le nv
 \quad\text{for every }(\theta,\beta).
 \label{eq:cal-adaptive-caps}
\end{equation}
The tower property gives the same bounds for the sums of $u_i,v_i$.
The two hypotheses in~\eqref{eq:cal-budget-pair} depend only on the
prescribed caps, not on the controller's realized or parameter-dependent
allocation. In the KL chain rule, use $4u_i\le1$ for the surviving
terms and the two expectation bounds for the remaining terms. This gives
\begin{equation}
 \KL(P_-^{\mathcal A}\Vert P_+^{\mathcal A})
 \le125n\varepsilon^2Q_{u,v}.
 \label{eq:cal-adaptive-kl}
\end{equation}
For $n\le(1000\varepsilon^2Q_{u,v})^{-1}$, the testing argument leading
to~\eqref{eq:cal-testing-bound} gives worst-case failure probability at
least $3/8$. This proves the lower bound for adaptive rules under the
expected category budgets.

\subsection{An implementable estimator and its variance}
\label{app:cal-estimator}

For the upper bound, independently repeat the fixed allocation
in~\eqref{eq:cal-allocation}. The following estimator subtracts the
calibration mean from the early-read mean; it does not use $\beta$.
For $u,v>0$, define
\begin{equation}
 T_E=\frac{R_E}{u}\left(Y-\frac14-\frac{d_K}{2}\right),\qquad
 T_C=\frac{R_C}{v}\left(Y-\frac14-\frac{q_j}{2}\right),\qquad
 D=2(T_E-T_C).
 \label{eq:cal-difference-estimator}
\end{equation}
Each residual is evaluated only on its indicated read branch; the term
is zero otherwise. The indices $K$ and $j$ are the recorded first-action
labels. Selection is independent of $O_2$, so on a calibration branch
the measured bit retains its prior $\Ber(q_j)$. It follows that
\begin{equation}
 \E T_E=\beta+\theta/2,\qquad
 \E T_C=\beta,\qquad \E D=\theta.
 \label{eq:cal-difference-mean}
\end{equation}
The two residuals have absolute value at most one, and $R_ER_C=0$.
Consequently,
\begin{equation}
 \operatorname{Var}(D)\le\E D^2
 =4\E(T_E^2+T_C^2)
 \le4(1/u+1/v)=4/h_{u,v}.
 \label{eq:cal-difference-variance}
\end{equation}
The original log gives a second unbiased estimate of $\theta$:
\begin{equation}
 P=\begin{cases}(B-d_K)/q,&E_H,\\0,&E_H^c.
 \end{cases}
 \qquad \E P=\theta,\qquad \operatorname{Var}(P)\le1/q.
 \label{eq:cal-passive-estimator}
\end{equation}
Here $B=O_H$ on $E_H$, so every quantity used by the estimator is
observed. For $h_{u,v}>0$, combine the two estimates as
\begin{equation}
 T=\frac{qP+h_{u,v}D}{Q_{u,v}}.
 \label{eq:cal-combined-estimator}
\end{equation}
They can be correlated on the same trajectory. Retaining the covariance
term and applying Cauchy--Schwarz gives
\begin{equation}
 \begin{aligned}
 \operatorname{Var}(T)
 &=\frac{q^2\operatorname{Var}(P)+h_{u,v}^2\operatorname{Var}(D)
       +2qh_{u,v}\operatorname{Cov}(P,D)}{Q_{u,v}^2}\\
 &\le\frac{(\sqrt q+2\sqrt{h_{u,v}})^2}{Q_{u,v}^2}
 \le\frac5{Q_{u,v}}.
 \end{aligned}
 \label{eq:cal-combined-variance}
\end{equation}
If $u=0$ or $v=0$, take $T=P$; the same final bound holds.
For independent trajectories, use
\begin{equation}
 \widehat J=\frac9{20}+
 \operatorname{clip}_{[0,1/10]}
 \left(\frac1n\sum_{i=1}^nT_i\right).
 \label{eq:cal-value-estimator}
\end{equation}
Clipping cannot increase error. Chebyshev's inequality yields
\begin{equation}
 \sup_{\theta,\beta}
 \Pr\bigl(|\widehat J-J_\theta|>\varepsilon\bigr)
 \le\frac5{n\varepsilon^2Q_{u,v}}\le\frac13
 \quad\text{if }n\ge\frac{15}{\varepsilon^2Q_{u,v}}.
 \label{eq:cal-upper-bound}
\end{equation}
Together with~\eqref{eq:cal-adaptive-kl}, this proves
Theorem~\ref{thm:readout-calibration}, including allocations with one or
both read budgets equal to zero.

\subsection{Allocation under a total expected budget}
\label{app:cal-total-budget}

Suppose only the total expected number of reads is bounded by $nb$,
under every parameter pair. Set $B_b=\min\{b,1\}$. The complementary
term in~\eqref{eq:cal-information-scale} has maximum
\begin{equation}
 h_\star(b)=
 \max_{\substack{0\le u\le1/4,\;0\le v\le3/4\\u+v\le b}}
 \frac{uv}{u+v}
 =\begin{cases}
 B_b/4,&0\le B_b\le1/2,\\[2pt]
 1/4-1/(16B_b),&1/2<B_b\le1,
 \end{cases}
 \label{eq:cal-optimal-scale}
\end{equation}
where the ratio is zero at $u=v=0$. When both budgets are positive,
increasing either one increases the ratio. An optimum therefore uses
total budget $B_b$. For fixed
$u+v=B_b$, the product $uv$ is maximized by equal allocation until
the early-read capacity $1/4$ is reached. An attaining allocation is
\begin{equation}
 (u_\star,v_\star)=
 \begin{cases}
 (B_b/2,B_b/2),&B_b\le1/2,\\
 (1/4,B_b-1/4),&B_b>1/2.
 \end{cases}
 \label{eq:cal-rate-optimal-allocation}
\end{equation}

\begin{corollary}[Total-budget calibration rate]
\label{cor:cal-total-budget}
Under the step-$2$ calibration protocol, with total expected read count
at most $nb$ for every $(\theta,\beta)$, the minimax trajectory complexity
at success probability $2/3$ is
\begin{equation}
 n_\star^{\rm cal}(\varepsilon,b)
 =\Theta\!\left(\frac1{\varepsilon^2[q+h_\star(b)]}\right),
 \label{eq:cal-total-budget-rate}
\end{equation}
for $H\ge3$ and $0<\varepsilon\le1/40$. The lower bound allows
adaptation across trajectories. The fixed allocation
in~\eqref{eq:cal-rate-optimal-allocation} attains the rate.
\end{corollary}

\begin{proof}
The upper bound follows from~\eqref{eq:cal-upper-bound}. To prove the
adaptive lower bound, choose a direction before observing any data:
\[
 \alpha_b=\begin{cases}
 1/2,&0\le B_b\le1/2,\\
 1/(4B_b),&B_b>1/2.
 \end{cases}
\]
For every feasible average allocation $(U,V)$,
\begin{equation}
 U(1-\alpha_b)^2+V\alpha_b^2\le h_\star(b).
 \label{eq:cal-uniform-direction}
\end{equation}
When $B_b\le1/2$, the left side is $(U+V)/4$.
When $B_b>1/2$, the coefficient of $U$ is at least that of $V$;
its linear maximum under the capacity and total-budget constraints
occurs at $U=1/4$, $V=B_b-1/4$. Substitution gives
$1/4-1/(16B_b)$.

Use~\eqref{eq:cal-budget-pair} with $\alpha=\alpha_b$. For any adaptive
controller, let $U,V$ be its expected average read counts under the
negative hypothesis. These obey the capacity constraints and
$U+V\le B_b$. Applying the conditional KL chain rule as before and
then~\eqref{eq:cal-uniform-direction} gives
$\KL(P_-^{\mathcal A}\Vert P_+^{\mathcal A})
 \le125n\varepsilon^2[q+h_\star(b)]$.
The same two-point testing bound completes the proof.
\end{proof}

The allocation in~\eqref{eq:cal-rate-optimal-allocation} optimizes the
rate term $uv/(u+v)$ exactly. It is therefore rate-optimal; the theorem
does not identify the allocation minimizing the exact finite-sample
risk. In particular, equal positive constant budgets provide an $H$-independent
rate, whereas spending the same budget on either type alone leaves the
factor $q^{-1}$.

\subsection{A finite experiment equivalent to the complete log}
\label{app:cal-sufficient}

The numerical comparison uses the complete-log likelihood through a
finite sufficient statistic. We give its factorization here for the
fixed allocation~\eqref{eq:cal-allocation}. Define the observable event
$S=\{A_2,\ldots,A_{H-1}\text{ all preserve}\}$, with probability $s$.
From each record retain the first action and its label, $S$, the read
indicator $R$, and $Y$ if read. After a first preserve, additionally
retain $B=O_H$ only on $S$. After a first reset, retain $O_2$, and
additionally $B'=O_H$ only on $S$. Thus $B$ and $B'$ are retained only
when observed at the terminal step.

Write $\mathrm b(p;x)=p^x(1-p)^{1-x}$ for a Bernoulli mass. For
$g\in\{c,r\}$, set
\[
 \rho_c=1/8,\quad \rho_r=3/8,\qquad
 \lambda_c=4u,\quad \lambda_r=4v/3,\qquad
 W_{g,S,R}=\rho_g\,\mathrm b(s;S)\,\mathrm b(\lambda_g;R).
\]
The first-action probabilities $\rho_g$ refer to each label separately.
Zero-probability categories are omitted. The category masses fall into
four cases; a sensor factor is absent when $R=0$.

\paragraph{First preserve, surviving bit.}
For $A_1=c_k$ and $S=1$, the mass is
\begin{equation}
 W_{c,1,R}\,\mathrm b(d_k+\theta;B)
 \big[\mathrm b(1/4+B/2+\beta;Y)\big]^R.
 \label{eq:cal-category-preserve-survive}
\end{equation}
The terminal observation supplies $B$, so this retains the joint
information in $(B,Y)$.

\paragraph{First preserve, overwritten bit.}
For $A_1=c_k$ and $S=0$, the mass is
\begin{equation}
 W_{c,0,R}
 \big[\mathrm b(1/4+d_k/2+\theta/2+\beta;Y)\big]^R.
 \label{eq:cal-category-preserve-reset}
\end{equation}
Before the first later reset, all ordinary observations are blank.
After that reset, the original suffix is independent of the seed.
Thus $Y$ carries all surviving parameter dependence on this branch.

\paragraph{First reset, surviving bit.}
Let $e_o(b)=\mathrm b(1/10+4b/5;o)$ be the original reset emission
mass. For $A_1=r_j$ and $S=1$, the category mass is
\begin{equation}
 W_{r,1,R}\,\mathrm b(q_j;B')\,e_{O_2}(B')
 \big[\mathrm b(1/4+B'/2+\beta;Y)\big]^R.
 \label{eq:cal-category-reset-survive}
\end{equation}
Conditional on $B'$, the ordinary observation and sensor noise are
independent; their marginal correlation is retained in this expression.

\paragraph{First reset, overwritten bit.}
For $A_1=r_j$ and $S=0$, the only original observation of the step-$2$
bit is $O_2$. Define
\begin{equation}
 m_j(o)=(1-q_j)e_o(0)+q_je_o(1),\qquad
 w_j(o)=\frac{q_je_o(1)}{m_j(o)}.
 \label{eq:cal-known-posterior}
\end{equation}
The category mass is
\begin{equation}
 W_{r,0,R}\,m_j(O_2)
 \big[\mathrm b(1/4+w_j(O_2)/2+\beta;Y)\big]^R.
 \label{eq:cal-category-reset-reset}
\end{equation}
For $(j,o)=(0,0),(0,1),(1,0),(1,1)$, respectively,
$w_j(o)=1/82,1/2,1/2,81/82$. Using the unconditional reset prior
$q_j$ in place of this posterior would discard information in the
complete log.

In all four cases, the omitted actions and observations have a
parameter-independent conditional distribution given the category.
Preserve commands retain the public memory and produce blank
observations; the first subsequent reset erases the relevant bit;
the remaining kernels are known. The terminal action is also known
and parameter independent. If $C(Z)$ denotes the retained category,
there is therefore one parameter-independent kernel $K$ such that
\begin{equation}
 P_{\theta,\beta}(Z=z)
 =p_{\theta,\beta}(C(z))K(z\mid C(z)).
 \label{eq:cal-sufficiency-factorization}
\end{equation}
One may take $K$ to be the conditional complete-record law at
$(\theta,\beta)=(1/20,0)$; all parameter pairs have the same positive
support. Conversely, $C$ is a deterministic function of the record.
Sampling the finite categories is thus equivalent to sampling the
complete statistical experiment.

Each category mass is a known weight times at most two Bernoulli
factors, and every success probability is affine in $(\theta,\beta)$.
The joint log likelihood is consequently concave on the parameter
rectangle. A joint category mass itself can contain the product
$\theta\beta$; its Bernoulli factors give the concavity.
This factorization permits the same joint likelihood estimator under
both allocations while retaining rare terminal-bit information and
ordinary reset observations.

\subsection{Equal-budget comparison}\label{app:cal-experiment}

We compare early-only readout, $(u,v)=(1/4,0)$, with early readout plus
calibration, $(u,v)=(1/8,1/8)$. Both spend one measurement per four
trajectories in expectation. Read selection occurs at $h=2$, independently
of $O_2$, and the logger does not use the measurement. The two designs share
all other kernels. Both jointly estimate $(\theta,\beta)$ by maximum
likelihood over $[0,1/10]\times[-1/20,1/20]$ using the complete observed log.
We sample the exactly sufficient categories in
Appendix~\ref{app:cal-sufficient}; no unobserved bit is supplied to the
estimator. The negative log likelihood is convex. If its minimum is not
unique, both designs select the minimum-norm point in the centered
coordinates $(\theta-1/20,\beta)$.

The main comparison uses $(\theta,\beta)=(0.09,-0.02)$ and $(0.01,0.02)$,
which have the same early-read mean given the first preserve label.
Figure~\ref{fig:calibration-budget} reports the square root of their equally
weighted MSE. The data-free estimator $\widehat J=1/2$ has RMSE $0.04$
at both points. At $H=8$ and $n=20{,}000$, the MLE's RMSE is $0.0555$ for early-only
readout and $0.0196$ with calibration. With calibration, the error remains
near $0.02$ for $H\in\{6,\ldots,10\}$ and decreases with sample size.

We also report $(\theta,\beta)=(0.07,0.02)$ and $(0.05,0)$ in
Figure~\ref{fig:calibration-sensitivity}. All four parameter settings were
fixed before inspecting the results. Each parameter, design, and setting
uses $200$ independent datasets, giving $20{,}800$ fits across the two
grids. Pointwise $95\%$ Monte Carlo intervals use $4{,}000$ bootstrap
resamples of squared errors, stratified by parameter for the two-point
average. These intervals describe simulation uncertainty, not confidence
intervals for the unknown target value.

The early-only MLE's RMSE need not increase monotonically with $H$. When rare surviving
seeds disappear from a dataset, the likelihood can have a flat direction;
the finite parameter box and the common tie rule then affect pointwise
risk. At the center, this rule can favor early-only readout. The plots thus
compare a concrete estimator at specified parameters, while
Theorem~\ref{thm:readout-calibration} characterizes uniform estimation over
the parameter box.

The supplement retains every replicate's category counts and fit diagnostics.
Exact rational
enumeration checks the category likelihood against $38{,}464$ complete-record
atoms for $H\in\{3,4,5\}$. All fits pass the constrained first-order checks.
On a fixed subset of $416$ fits, independent SciPy optimization matches
the objective values to numerical precision and meets the same tolerance.
Parameter estimates can differ along flat likelihood directions.
The supplied scripts reproduce the simulation, uncertainty intervals,
and figures from the recorded seed.

\begin{figure}[t]
  \centering
  \includegraphics[width=\textwidth]{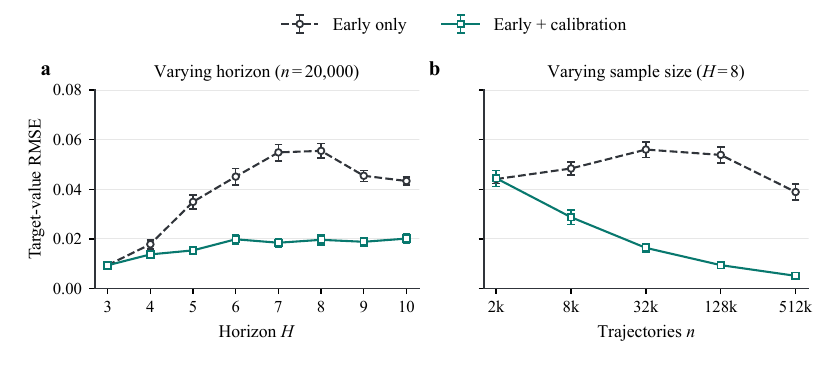}
  \caption{\textbf{Equal readout budgets with unknown sensor bias.}
  Both designs use complete-log joint maximum likelihood and expected budget
  $1/4$. Points show the square root of the equally weighted MSE at
  $(\theta,\beta)=(0.09,-0.02)$ and $(0.01,0.02)$, with $200$ independent
  datasets per parameter and setting. Bars are pointwise $95\%$ Monte Carlo
  intervals from $4{,}000$ stratified bootstrap resamples.
  (a) Varying $H$ at $n=20{,}000$. (b) Varying $n$ at $H=8$.
  Lines guide the eye.}
  \label{fig:calibration-budget}
\end{figure}

\begin{figure}[p]
  \centering
  \includegraphics[width=\textwidth]{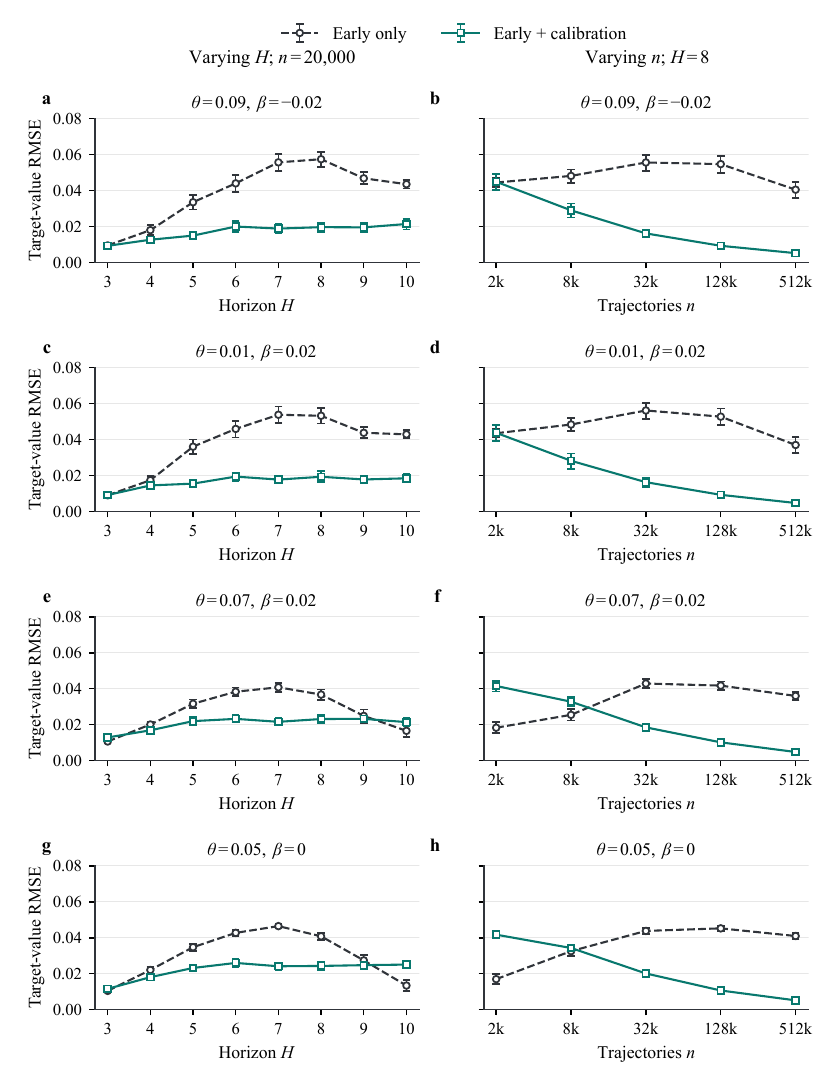}
  \caption{\textbf{All four parameter settings.}
  Rows show the indicated $(\theta,\beta)$ with the estimators and budgets
  of Figure~\ref{fig:calibration-budget}. The left column varies $H$ at
  $n=20{,}000$; the right varies $n$ at $H=8$. Points are RMSE over $200$
  independent datasets; bars are pointwise $95\%$ Monte Carlo intervals
  from $4{,}000$ bootstrap resamples. The center case illustrates that
  improved uniform rates do not imply smaller risk at every parameter.}
  \label{fig:calibration-sensitivity}
\end{figure}

\clearpage
\section{History-conditioned value representations}\label{vc:appendix}

We prove Theorem~\ref{thm:value-calibration} by comparing scalar Bellman residuals. The argument does not recover the observation or transition kernels. We then prove Proposition~\ref{prop:value-gap} for the common-operator construction and give an example in which value calibration is possible but full-state calibration is not.

\subsection{Assumptions and an explicit bound}\label{vc:assumptions}

Let $\mathcal M$ be a known finite model class with $K=|\mathcal M|$, a common known logger, and a common known memoryless target. State, observation, and action coordinates are fixed across models. We use the timing $S_h\to O_h\to A_h\to R_h\to S_{h+1}$, with common known rewards $R_h=r_h(O_h,A_h)\in[0,1]$. Given $(S_h,A_h)$, the next state is independent of the earlier history and current observation. There is no transition after the terminal reward. Write $P_M=P_b^M$ for the complete logged law, including the terminal action. The reward coordinates are determined by the observation--action record and may be omitted without changing likelihood ratios or statistical distances.

Write $\tau_{h-1}=(O_1,A_1,\ldots,O_{h-1},A_{h-1})$ and
$F_h=(O_h,A_h,\ldots,O_{H-1},A_{H-1},O_H)$. Let
\begin{equation}\label{vc:weights}
 w_h=\frac{\pi_{e,h}(A_h\mid O_h)}{\pi_{b,h}(A_h\mid\tau_{h-1},O_h)},
 \qquad 0\le w_h\le C_A.
\end{equation}
The logger has full action support, with inverse probabilities bounded by $C_A$. In every model,
$\E_M[w_h\mid\tau_{h-1},O_h]=1$.

For each candidate $N$ and each $2\le h\le H$, suppose there is a fixed observable scalar function $g_h^N(\tau_{h-1},F_h)$ such that
\begin{equation}\label{vc:calibration}
 \E_M[g_h^N\mid S_h=s,\tau_{h-1}=\tau]=V_h^N(s)
 \qquad\text{for every }M,N\in\mathcal M.
\end{equation}
The identity is required on positive-probability conditioning events. Here $V_h^N$ is the target continuation value in candidate $N$. The function may depend on $N$, but the same function must satisfy the identity in every generating model $M$. Suppose
\begin{equation}\label{vc:moment}
 \max_{M,N\in\mathcal M}\max_{2\le h\le H}\E_M[(g_h^N)^2]\le C_V.
\end{equation}
Only a global second-moment bound is required. The functions may use the observed past and may take negative values.

Let $b_h^M(\tau)=\Pp_M(S_h=\cdot\mid\tau)$ and
$\Sigma_{H,h}^M=\E_M[b_h^M(b_h^M)^\top]$. Assume
\begin{equation}\label{vc:coverage}
 \lambda_{\min}(\Sigma_{H,h}^M)\ge C_{\rm hist}^{-1},
 \qquad M\in\mathcal M,\quad 2\le h\le H.
\end{equation}
We include the terminal cut and impose no coverage condition at the empty initial history.

Use the unnormalized Hellinger convention
$\mathsf h(P,Q)^2=\sum_z(\sqrt{P(z)}-\sqrt{Q(z)})^2$.
An explicit version of Theorem~\ref{thm:value-calibration} is
\begin{equation}\label{vc:lipschitz}
 |J_M-J_N|\le L\,\mathsf h(P_M,P_N),
\end{equation}
where
\begin{equation}\label{vc:constant}
\begin{aligned}
 A_0&=\sqrt{C_A}+C_A\sqrt{C_V},\\
 A_{\mathrm{int}}&=\sqrt{C_A}+(C_A+1)\sqrt{C_V},\\
 A_{\mathrm{end}}&=\sqrt{C_A}+\sqrt{C_V},\\
 L&=2A_0+2C_{\rm hist}\bigl[(H-2)A_{\mathrm{int}}+A_{\mathrm{end}}\bigr].
\end{aligned}
\end{equation}
Consequently, finite-class maximum likelihood followed by evaluation of the selected model satisfies
\begin{equation}\label{vc:sample}
 \Pp_{M_\star}\!\left(
 |\widehat J-J_{M_\star}|>
 L\sqrt{\frac{2\log(K/\delta)}n}\right)\le\delta.
\end{equation}
Since $C_A,C_{\rm hist}\ge1$, we have
$L\le2HC_{\rm hist}C_A(1+2\sqrt{C_V})$ and
$(1+2\sqrt{C_V})^2\le5(1+C_V)$. Thus the sufficient sample size
\begin{equation}\label{vc:simple-sample}
 n\ge\frac{40H^2C_{\rm hist}^2C_A^2(1+C_V)}{\epsilon^2}
 \log\frac K\delta
\end{equation}
gives error at most $\epsilon$ with probability at least $1-\delta$.

\subsection{Scalar residual comparison}\label{vc:residual-proof}

For a common vector-valued statistic $Z$, suppose the squared Euclidean norm
has expectation at most $A^2$ under both $P$ and $Q$. Cauchy--Schwarz gives
\begin{equation}\label{vc:hellinger-moment}
 \|\E_PZ-\E_QZ\|_2\le2A\,\mathsf h(P,Q).
\end{equation}
Indeed, factor $P-Q=(\sqrt P-\sqrt Q)(\sqrt P+\sqrt Q)$ and use
$(\sqrt P+\sqrt Q)^2\le2(P+Q)$. The inequality does not require mutual absolute continuity.

Let $K_h^M$ and $\rho_h^M$ be the state transition and immediate reward induced by the target:
\begin{equation}\label{vc:target-kernels}
\begin{aligned}
 K_h^M(s',s)&=\sum_o O_h^M(o\mid s)
                  \sum_a\pi_{e,h}(a\mid o)T_{a,h}^M(s'\mid s),\\
 \rho_h^M(s)&=\sum_o O_h^M(o\mid s)
                  \sum_a\pi_{e,h}(a\mid o)r_h(o,a).
\end{aligned}
\end{equation}
Memorylessness of the target makes these quantities independent of the past. Set $g_{H+1}^N=V_{H+1}^N=0$, omitting the transition term at the terminal step. For fixed $M,N$, define
\begin{equation}\label{vc:residuals}
\begin{aligned}
 d_h^{M,N}&=\rho_h^M+(K_h^M)^\top V_{h+1}^N-V_h^N,\\
 R_{N,h}&=w_h\{r_h+g_{h+1}^N\}-g_h^N,
 \qquad 2\le h\le H.
\end{aligned}
\end{equation}
For $h<H$, first condition $g_{h+1}^N$ on $(S_{h+1},\tau_h)$ and apply~\eqref{vc:calibration}; then importance-weight the current action. This is valid because $w_h$ is measurable with respect to $\tau_h$. Together with the analogous reward identity, it yields
\begin{equation}\label{vc:conditional-residual}
 \E_M[R_{N,h}\mid\tau_{h-1}]
 =(d_h^{M,N})^\top b_h^M(\tau_{h-1}).
\end{equation}
The same formula holds at $h=H$ without a next-value term. In model $N$, its Bellman equation gives
$\E_N[R_{N,h}\mid\tau_{h-1}]=0$.

Extend $b_h^M$ arbitrarily as a probability vector on histories with zero $P_M$ probability. It is then a fixed observable instrument of norm at most one, and
\begin{equation}\label{vc:instrument}
 \E_M[b_h^M R_{N,h}]=\Sigma_{H,h}^M d_h^{M,N},
 \qquad \E_N[b_h^M R_{N,h}]=0.
\end{equation}
The second identity uses conditional centering under $N$; it does not require the instrument to equal $N$'s belief.

Under any candidate logged law $P$, bounded rewards and $w_h^2\le C_Aw_h$ imply
$\E_P[(w_hr_h)^2]\le C_A$. Moreover,
$\E_P[(w_hg_{h+1}^N)^2]\le C_A^2C_V$.
These bounds use only the global moment in~\eqref{vc:moment}. The $L_2$ triangle inequality gives
\begin{equation}\label{vc:residual-moment}
 \|R_{N,h}\|_{L_2(P)}\le
 \begin{cases}
 A_{\mathrm{int}},&2\le h<H,\\
 A_{\mathrm{end}},&h=H.
 \end{cases}
\end{equation}
Multiplication by $b_h^M$ cannot increase this bound. Applying~\eqref{vc:hellinger-moment} to~\eqref{vc:instrument} and then using~\eqref{vc:coverage}, we obtain
\begin{equation}\label{vc:residual-bound}
 \|d_h^{M,N}\|_\infty\le\|d_h^{M,N}\|_2
 \le2C_{\rm hist} A_h\,\mathsf h(P_M,P_N),
\end{equation}
with the appropriate $A_h$ from~\eqref{vc:residual-moment}.

\subsection{The first step and finite-class estimation}\label{vc:initial-proof}

Consider the observable statistic
$Z_N=w_1\{r_1+g_2^N\}$. Calibration at step $2$ implies
\begin{equation}\label{vc:first-step}
 \E_MZ_N=\E_{M,\pi_e}[r_1+V_2^N(S_2)],
 \qquad \E_NZ_N=J_N.
\end{equation}
This handles an arbitrary initial state distribution without an initial-state decoder or an invertible empty-history Gram. The same moment calculation gives $\|Z_N\|_{L_2(P)}\le A_0$ under every candidate. Hence
$|\E_MZ_N-J_N|\le2A_0\mathsf h(P_M,P_N)$.

Along a target trajectory in $M$, the value terms in~\eqref{vc:residuals} telescope:
\begin{equation}\label{vc:telescope}
 J_M-\E_MZ_N
 =\sum_{h=2}^H\E_{M,\pi_e}[d_h^{M,N}(S_h)].
\end{equation}
Combining this identity with~\eqref{vc:residual-bound} proves~\eqref{vc:lipschitz} with the constant in~\eqref{vc:constant}.

For completeness, let $\widehat M$ maximize likelihood among the $K$ candidates. Under a true model $M_\star$, the probability that a fixed $N$ has likelihood at least that of $M_\star$ is at most
\begin{equation}\label{vc:likelihood}
 \left(\sum_z\sqrt{P_{M_\star}(z)P_N(z)}\right)^n
 =\left(1-\frac{\mathsf h(P_{M_\star},P_N)^2}{2}\right)^n
 \le e^{-n\mathsf h(P_{M_\star},P_N)^2/2}.
\end{equation}
This is Markov's inequality applied to the square root of the likelihood ratio. A union bound and~\eqref{vc:lipschitz} prove~\eqref{vc:sample}. The estimator uses candidate logged laws and candidate values. The common functions in~\eqref{vc:calibration} certify the class's value--distribution relation; they need not be learned or supplied from the unknown generating model.

\subsection{Value representation without full-state recovery}\label{vc:strict-example}

A common state decoder $\psi_h$ immediately gives value functions
$g_h^N=(V_h^N)^\top\psi_h$. The converse need not hold, even with full-rank history coverage and different target values.

Take $H=2$, a single initial state, and four first-step actions $a_{i,j}$ with $i,j\in\{0,1\}$. The logger is uniform. The next state is $(X,Z)\in\{0,1\}^2$, with conditionally independent coordinates given the first action. In model $\mathsf{M}_0$, $X\sim\Ber(1/4+i/2)$; in $\mathsf{M}_1$, $X\sim\Ber(3/4-i/2)$. Both use $Z\sim\Ber(1/4+j/2)$. Set $R_1=0$. The terminal observation and reward are $O_2=R_2=X$. Terminal actions are uniform and have no effect. The target chooses $i=1$ and is uniform in $j$, so
$J_{\mathsf{M}_0}=3/4$ and $J_{\mathsf{M}_1}=1/4$.

For either candidate, $V_2^N(x,z)=x$. Thus the same $g_2^N=O_2$ satisfies~\eqref{vc:calibration}, with $C_V=1/2$ and $C_A=4$. Put
$W=\left(\begin{smallmatrix}3/4&1/4\\1/4&3/4\end{smallmatrix}\right)$.
The four history beliefs form $W\otimes W$, up to permutation of the first-action labels. Their Gram is
$\frac14(WW^\top)\otimes(WW^\top)$, with minimum eigenvalue $1/64$, so $C_{\rm hist}=64$.
Yet no full-state history-calibrated decoder exists. At a fixed first-action history and a fixed $X=x$, both values of $Z$ have positive probability and yield exactly the same observable future. The conditional means of a common decoder cannot equal both $e_{(x,0)}$ and $e_{(x,1)}$. Only the value-relevant coordinate needs to be recovered in this example.

\subsection{Representation cost in the common-operator construction}\label{vc:gap-proof}

In the construction of Appendix~\ref{ko:construction}, the target preserves every current bit. Therefore, for every candidate $N$ and every $h\ge2$,
\begin{equation}\label{vc:bit-value}
 V_h^N(\mathsf{H}_b)=V_h^N(\mathsf{R}_b)=b.
\end{equation}
Let $C_V^\star$ be the infimum of
$\max_{M\in\{-,+\},\,2\le h\le H}\E_M[g_h^2]$
over common functions calibrated to this bit at every history. Since the values in~\eqref{vc:bit-value} do not depend on $N$, allowing a separate family for each candidate cannot reduce this infimum.

\begin{proposition}[Marginal and history-conditioned value calibration]\label{prop:value-gap}
In the two-model class of Proposition~\ref{prop:known-operators}, there are
common marginal value functions $g_h^{\rm marg}(G_h)$, where
$G_h=(O_h,\ldots,O_H)$, such that
\begin{equation}\label{vc:marginal-value-statement}
 \E_M[g_h^{\rm marg}\mid S_h=s]=V_h^N(s),\qquad
 \max_{M,\,2\le h\le H}\E_M[(g_h^{\rm marg})^2]<2660.
\end{equation}
The optimal history-conditioned cost satisfies
\begin{equation}\label{vc:gap-statement}
 \frac{113}{256}+\frac{15}{256}4^{H-2}
 \le C_V^\star\le
 \frac{35}{64}+\frac{15}{256}4^{H-2}.
\end{equation}
\end{proposition}

\paragraph{An erasure experiment at a fixed history.}
Fix a first-preserve history $\tau=(\mathrm{start}_1,c_{k,m})$ and write $r=4^{-(H-2)}$. Conditional on $S_2=\mathsf{H}_b$ and this history, the full future law is
\begin{equation}\label{vc:erasure}
 Q_b=(1-r)R+rT_b.
\end{equation}
The three laws $R,T_0,T_1$ have pairwise disjoint supports. A subsequent reset erases the seed, giving the same $R$ for either bit; otherwise the terminal observation reveals it. These conditional future laws are also common across models.

For a calibrated scalar $g_2$, set $c=\E_Rg_2$ and $m_b=\E_{T_b}g_2$. Calibration requires
$(1-r)c+rm_0=0$ and $(1-r)c+rm_1=1$.
If the seed prior at this history is $p$, Jensen's inequality gives
\begin{equation}\label{vc:erasure-moment}
\begin{aligned}
 \E[g_2^2\mid\tau]
 &\ge(1-r)c^2+r\{(1-p)m_0^2+pm_1^2\}\\
 &=p^2+\frac{p(1-p)}r+\frac{1-r}{r}(c-p)^2
 \ge p^2+\frac{p(1-p)}r.
\end{aligned}
\end{equation}
Infinite second moments already satisfy the desired lower bound, so all means above may be assumed finite.

\paragraph{A lower bound for common functions.}
Average the two model laws at each fixed first-preserve history. Because $g_2$ is common and the conditional laws in~\eqref{vc:erasure} are common, the averaged seed prior is
$c_m=1/2+(2m-1)/8$.
Each of the four histories has probability $1/16$ in both models. Their contribution to the average second moment is therefore at least
\begin{equation}\label{vc:preserve-lower}
 \frac1{16}\sum_{k,m}\left(c_m^2+\frac{c_m(1-c_m)}r\right)
 =\frac{17}{256}+\frac{15}{256r}.
\end{equation}
The initial-reset histories add at least $3/8$: their total mass is $3/4$, their bit marginal is uniform, and conditional calibration gives $\E[g_2^2\mid S_2,\tau]\ge b^2=b$. Since a maximum is at least an average,
\begin{equation}\label{vc:cost-lower}
 C_V^\star\ge\frac{113}{256}+\frac{15}{256}\,4^{H-2}.
\end{equation}

\paragraph{A common construction with the same leading term.}
Let $E_h$ be the event that all actions from $h$ through $H-1$ preserve, let $r_h=4^{-(H-h)}$, and write $Y\in\{0,1\}$ for the terminal bit. The previous action identifies the current hidden/reset type. On a hidden history before step $H$, take
\begin{equation}\label{vc:centered-bridge}
 g_h=c_h(\tau)+\frac{\one\{E_h\}}{r_h}\{Y-c_h(\tau)\},
 \qquad
 c_h(\tau)=
 \begin{cases}
 c_m,&h=2\text{ and the first action writes }m,\\
 1/2,&h\ge3.
 \end{cases}
\end{equation}
On a reset history before step $H$, use $g_h=(10X_h-1)/8$, where $X_h\in\{0,1\}$ is the current noisy observation. Finally set $g_H=Y$.
All functions are observable and common. Conditional on any state and history, $E_h$ has probability $r_h$ and reveals the current bit at the terminal step. The centered function in~\eqref{vc:centered-bridge} is therefore calibrated. The reset function is calibrated because its observation has mean $1/10+(4/5)b$.

For a reset bit $b$, its function has second moment $b+9/64$. The reset type has probability $3/4$ with uniform bit marginal, contributing $123/256$ at every preterminal cut. At $h=2$, each hidden-type event with writer $m$ has probability $1/8$ and bit mean $c_m$. Direct substitution gives
\begin{equation}\label{vc:cost-upper2}
 \E_M[g_2^2]
 =\frac{123}{256}+\frac{17}{256}+\frac{15}{256r}
 =\frac{35}{64}+\frac{15}{256r}.
\end{equation}
For $3\le h<H$, the hidden type has mass $1/4$ and uniform bit marginal, so
$\E_M[g_h^2]=139/256+1/(16r_h)$.
Here $r_h\ge4r$, making this moment smaller than~\eqref{vc:cost-upper2}. At $h=H$, the moment is $1/2$. Combining with~\eqref{vc:cost-lower} yields
\begin{equation}\label{vc:cost-sandwich}
 \frac{113}{256}+\frac{15}{256}4^{H-2}
 \le C_V^\star\le
 \frac{35}{64}+\frac{15}{256}4^{H-2}.
\end{equation}
The additive gap is $27/256$, independent of $H$.

\paragraph{The same value is cheap to calibrate marginally.}
Let $U_h^{\mathrm{obs}}$ be the common marginal operator, and define
\begin{equation}\label{vc:weighted-left-inverse}
 W_h=\operatorname{diag}(U_h^{\mathrm{obs}}\mathbf1),\qquad
 \Sigma_h=(U_h^{\mathrm{obs}})^\top W_h^{-1}U_h^{\mathrm{obs}},\qquad
 D_h=\Sigma_h^{-1}(U_h^{\mathrm{obs}})^\top W_h^{-1},
\end{equation}
restricting inverses to positive row support. With $v=(0,1,0,1)^\top$, the common statistic
$g_h^{\mathrm{marg}}(G_h)=v^\top D_he_{G_h}$ satisfies
$\E_M[g_h^{\mathrm{marg}}\mid S_h=s]=v(s)$.
Since $\operatorname{diag}(U_h^{\mathrm{obs}}p_h)\preceq W_h$,
\begin{equation}\label{vc:marginal-cost}
 \E_M[(g_h^{\mathrm{marg}})^2]
 \le v^\top\Sigma_h^{-1}v
 \le2\|\Sigma_h^{-1}\|_1
 \le2C_F^{\mathrm{obs}}<2660,
 \qquad 2\le h\le H-1.
\end{equation}
The terminal bit itself handles $h=H$. Equations~\eqref{vc:cost-sandwich} and~\eqref{vc:marginal-cost} concern the same target continuation value. Their exponential separation comes from retaining the history in the calibration identity, not from recovering value-irrelevant state coordinates.

\subsection{A finite-dimensional certificate}\label{vc:computable-certificate}

For finite state, action, and observation spaces, the representation condition can be checked from the specified candidate models. Fix a cut $2\le h\le H$ and define
\begin{equation}\label{vc:certificate-support}
\begin{aligned}
 \mathcal Z_h
 &=\{(\tau,f):P_M(\tau_{h-1}=\tau,F_h=f)>0
                    \text{ for some }M\in\mathcal M\},\\
 \mathcal R_h
 &=\{(M,s,\tau):M\in\mathcal M,\ 
                    P_M(S_h=s,\tau_{h-1}=\tau)>0\}.
\end{aligned}
\end{equation}
Use $\mathcal R_h$ as the row index set and $\mathcal Z_h$ as the column index set of a matrix $\mathsf A_h$. For each candidate $N$, define a vector $v_h^N$ and diagonal matrices $D_h^M$ by
\begin{equation}\label{vc:certificate-matrices}
\begin{aligned}
 (\mathsf A_h)_{(M,s,\tau),(\tau',f)}
   &=\mathbf1\{\tau=\tau'\}
       P_M(F_h=f\mid S_h=s,\tau_{h-1}=\tau),\\
 (v_h^N)_{(M,s,\tau)}&=V_h^N(s),\\
 (D_h^M)_{(\tau,f),(\tau,f)}
   &=P_M(\tau_{h-1}=\tau,F_h=f).
\end{aligned}
\end{equation}
The matrix is block diagonal by history. Within each block, it stacks the conditional future laws across models and reachable states.

\begin{proposition}[A finite-dimensional representation certificate]
\label{prop:value-certificate}
Common functions satisfying~\eqref{vc:calibration} exist if and only if
$v_h^N\in\operatorname{col}(\mathsf A_h)$ for every $N\in\mathcal M$ and $2\le h\le H$.
The smallest global cost in~\eqref{eq:value-cost} is the value of the convex quadratically constrained program
\begin{equation}\label{vc:certificate-qcqp}
\begin{aligned}
 C_V^\star=\min_{\{x_h^N\},\,t}\quad &t\\
 \text{subject to}\quad
 &\mathsf A_hx_h^N=v_h^N
                  &&\text{for every }N,h,\\
 &(x_h^N)^\top D_h^M x_h^N\le t
                  &&\text{for every }M,N,h.
\end{aligned}
\end{equation}
The value is $+\infty$ if the linear constraints are infeasible; otherwise the minimum is attained.
Equivalently, let $\rho_\star$ be the minimum of $\rho\ge0$ subject to
\begin{equation}\label{vc:certificate-socp}
 \mathsf A_hx_h^N=v_h^N,\qquad
 \|(D_h^M)^{1/2}x_h^N\|_2\le\rho
 \quad\text{for every }M,N,h.
\end{equation}
This is a second-order cone program, and $C_V^\star=\rho_\star^2$.
\end{proposition}

\begin{proof}
Set $x_h^N(\tau,f)=g_h^N(\tau,f)$ on $\mathcal Z_h$. The conditional mean identity is exactly $\mathsf A_hx_h^N=v_h^N$, and
$\E_M[(g_h^N)^2]=(x_h^N)^\top D_h^M x_h^N$.
Conversely, a feasible vector defines an observable function on this union of supports; extend it arbitrarily outside the union. This proves both feasibility and the cost formula.

Each $D_h^M$ is positive semidefinite. Moreover, every diagonal entry of $\sum_M D_h^M$ is positive by the definition of $\mathcal Z_h$. A finite upper bound on all model-specific moments therefore bounds every coordinate of every $x_h^N$. If feasible, the program has a nonempty compact sublevel set and a continuous objective, so it attains its minimum. Replacing $t$ by $\rho^2$ gives the equivalent norm constraints.
\end{proof}

Every entry is determined by the candidate class and the known logger and target.
The matrices $D_h^M$ encode the global second moments in~\eqref{vc:moment}.
Enumerating histories and futures can require space exponential in $H$.
The program checks the sufficient representation condition, not a necessary
condition for OPE.

\section{A short probe with polynomial representation cost}\label{app:short_probe}

Fix any odd $H\ge3$. The first state and observation are fixed; at every later layer the physical state is a bit $B_h$. Even layers emit a blank, and odd layers $h\ge3$ reveal the bit. There are two actions, $\mathtt{probe}$ and $\mathtt{reset}$. The first action produces a bit with mean $1/4$ or $3/4$, respectively, and the logger chooses these actions equally often. At later preterminal steps, probe preserves the bit. Each entry of the model-dependent reset kernel belongs to $[1/4,3/4]$. Any finite collection of such kernels defines a candidate class. The common known logger can depend on the entire observed history and current observation, with both action probabilities in $[1/4,3/4]$. The target is any common memoryless policy, and all candidates share known observed reward functions in $[0,1]$.

For a candidate $N$, take $g_h^N=V_h^N(O_h)$ at visible layers and, at even layers, take
\begin{equation}\label{sp:bridge}
 g_h^N=
 \frac{\one\{A_h=\mathtt{probe}\}}
 {\pi_b(\mathtt{probe}\mid\tau_{h-1},\bot_h)}
 V_h^N(O_{h+1}).
\end{equation}
Observation time labels are suppressed inside $V_h^N$: the next observation reads the current bit, whose value is evaluated at time $h$. Fixing $(B_h=b,\tau_{h-1})$ fixes the probe probability. If probe is selected, $O_{h+1}=b$; otherwise the numerator is zero. Hence~\eqref{sp:bridge} has conditional mean $V_h^N(b)$ under every generating model. Its conditional second moment is at most $4\|V_h^N\|_\infty^2$. Since $h\ge2$, this gives $C_V\le4(H-1)^2$ and $C_A\le4$.

For completeness, history coverage also holds uniformly in the logger's memory and in the unknown reset kernels. With $v(p)=(1-p,p)^\top$, the second-layer history Gram is
\begin{equation}\label{sp:gram2}
 \Sigma_{H,2}=\tfrac12v(1/4)v(1/4)^\top+
 \tfrac12v(3/4)v(3/4)^\top
 =\frac1{16}\begin{pmatrix}5&3\\3&5\end{pmatrix}.
\end{equation}
Its minimum eigenvalue is $1/8$. At layer $3$, retain histories with a probe at layer $2$. Since $O_2$ is blank, this action does not select the hidden bit within a fixed first-action history. Thus $\Sigma_{H,3}\succeq\frac14\Sigma_{H,2}$, giving minimum eigenvalue at least $1/32$.

At every $t\ge3$, either bit has occupancy at least $1/16$: the preceding reset has conditional probability at least $1/4$, and its transition probability to either bit is at least $1/4$. At an even layer $h\ge4$, the preceding visible observation followed by a probe fixes the current bit. Each of these two events has probability at least $(1/16)(1/4)=1/64$, so $\Sigma_{H,h}\succeq I_2/64$. At an odd layer $h\ge5$, retain a visible bit at $h-2$ followed by two probes. Each such event has probability at least $(1/16)(1/4)^2=1/256$, giving $\Sigma_{H,h}\succeq I_2/256$. These bounds use successive conditioning and do not assume independent actions. They include the terminal layer, so $C_{\rm hist}\le256$ for all odd horizons.

Consequently Theorem~\ref{thm:value-calibration} gives a sufficient sample bound of
$O(H^4\epsilon^{-2}\log(|\mathcal M|/\delta))$.
Every even layer has current emission matrix $(1\;1)$ and two reachable states, so its current observation alone cannot reveal the bit. The target value also need not be known: choose two models whose later reset kernels produce $\Ber(1/4)$ and $\Ber(3/4)$, respectively. A target that always resets, with only terminal reward $B_H$, has values $1/4$ and $3/4$. A short preserving probe supplies the information needed to evaluate this target from the logs.

\end{document}